\documentclass[sigconf, nonacm]{acmart}

\usepackage{multirow}
\usepackage{booktabs}
\usepackage{subcaption}
\usepackage{placeins}  
\usepackage{tikz}
\usetikzlibrary{positioning, arrows.meta, shadows}

\makeatletter
\g@addto@macro\UrlBreaks{\do\-\do\/\do\_\do\.\do\:\do\,}
\g@addto@macro\UrlBigBreaks{\do\-\do\/\do\_\do\.}
\makeatother

\newcommand{\R}{\mathbb{R}}
\newcommand{\E}{\mathbb{E}}
\newcommand{\Prb}{\mathbb{P}}
\newcommand{\ip}[2]{\langle #1, #2 \rangle}
\newcommand{\norm}[1]{\|#1\|}

\newtheorem{theorem}{Theorem}

\newtheorem{lemma}{Lemma}

\begin{document}

\title{IVF-TQ: Calibration-Free Streaming Vector Search via a Codebook-Free Residual Layer}

\author{Tarun Sharma}
\affiliation{%
  \institution{Independent Researcher}
  \country{USA}
}
\email{tarun.sharma@ieee.org}

\renewcommand{\shortauthors}{Tarun Sharma}

\begin{abstract}
Approximate nearest neighbor (ANN) indexes deployed against streaming corpora silently lose recall over weeks. The standard diagnosis is distribution shift, but under shuffled-i.i.d.\ ingestion --- no shift at all --- product quantization still degrades $-3.8$pp at sub-matched bit budgets. The dominant production compression methods (PQ, OPQ, ScaNN) all fit a codebook to an initial sample and reuse it as the database grows by orders of magnitude.

This paper presents IVF-TQ, an inverted-file index whose residual compression layer is data-independent: a fixed random rotation followed by a precomputed Lloyd--Max scalar quantizer parameterised only by the bit width $b$ and dimension $d$. Only the IVF coarse $k$-means partition is trained. A uniform-over-sphere inner-product error bound (Theorem~\ref{thm:ip-uniform-main}) depending only on $(b, d, \delta)$ provides a structural guarantee no learned-codebook method admits. The same codebook-free design enables an IVF-amplification effect that closes the gap to Extended RaBitQ to within statistical noise ($+17.7$pp over flat TQ at matched bit budget), and an Adaptive variant that refreshes the partition without touching the compression layer. Across nine controlled cells (three 10M datasets, three PQ memory regimes, three seeds), per-batch PQ codebook retraining never recovers the streaming gap; IVF-PQ streaming stability requires per-dataset bit-budget tuning, while IVF-TQ holds at one fixed $(b, d)$ configuration on all three datasets with $\Delta \in [-0.80, +0.56]$pp. The contribution is operational: no codebook to train, no per-dataset bit-budget tuning, no retraining cycle that ever closes the gap.
\end{abstract}

\maketitle


\pagestyle{plain}
\thispagestyle{plain}

\section{Introduction}
\label{sec:intro}

Vector ANN indexes are increasingly deployed against \emph{streaming} corpora: retrieval-augmented generation pipelines~\cite{lewis2020rag} re-ingest hundreds of thousands of documents per day, recommendation feeds operate under continuous concept drift, and observability and security pipelines continuously stream new embeddings from heterogeneous sources. Practitioners observe that recall on these workloads silently degrades over weeks or months, often without any code changes.

The dominant compression methods used in production---product quantization~\cite{jegou2011pq}, Optimized PQ~\cite{ge2013opq}, and ScaNN's anisotropic learned quantization~\cite{guo2020scann}---all share the same property: each fits a codebook to the initial training sample. The standard remedy is periodic re-training, which is computationally expensive and, at the bit-matched memory regime most modern systems use, recovers no measurable recall on 10M streaming runs (paired-$t$ indistinguishable from no retraining in 8 of 9 cells) while costing $667$--$1328$s of cumulative compute per run (\S\ref{sec:streaming}, Tables~\ref{tab:streaming_deep10m_all}--\ref{tab:streaming_t2i10m_all}). Concurrently, \cite{adenali2025streamingquant} address the same problem by making PQ-style codebooks dynamically consistent with bounded I/O per update; we instead remove the codebook from the compression layer entirely.

The standard explanation for the streaming-recall drop is distribution shift. We make a different and more nuanced argument: PQ's streaming recall behavior is governed by \emph{bit budget} in a \emph{dataset-dependent} way. Across 9 controlled cells (three 10M datasets $\times$ three PQ memory budgets, three seeds each), shuffled-i.i.d.\ ingestion still produces material degradation at sub-matched bit budget, ruling out distribution shift as a sufficient explanation. What changes the picture is the bit budget itself: at sufficient capacity (dataset-dependent), PQ is streaming-stable; below it, even retraining cannot close the gap (\S\ref{sec:streaming_mechanism}). IVF-TQ's residual layer is governed by $(b, d)$ alone and avoids the calibration entirely.

\textbf{This paper.} We propose \emph{IVF-TQ}, an inverted-file index whose \emph{residual compression layer} is data-independent: a fixed random rotation followed by a precomputed Lloyd--Max scalar quantizer that depends only on the bit width and dimension. Only the IVF coarse partition is trained (by $k$-means); the residual quantizer never needs re-training. Building on TurboQuant~\cite{zandieh2025turboquant}, we trade $\sim$1pp of static recall at matched memory against the strongest learned-codebook baselines for streaming stability without per-dataset bit-budget tuning. We make three contributions:

\begin{enumerate}
    \item \textbf{IVF-TQ index architecture} (\S\ref{sec:method}): an inverted-file index in which residuals are compressed by a fixed random rotation + precomputed Lloyd--Max scalar quantization (data-independent residual layer; the IVF coarse partition is still trained by $k$-means), paired with a \emph{uniform-over-sphere inner-product error bound} (Theorem~\ref{thm:ip-uniform-main}) depending only on $(b, d, \delta)$ --- a structural guarantee no learned-codebook method (PQ, OPQ, ScaNN) admits. The bound is structural rather than numerically tight (full discussion in \S\ref{sec:ivfamp}); multi-seed empirical evidence supports the structural claim across nine controlled cells (\S\ref{sec:streaming}).
    \item \textbf{IVF amplification of codebook-free quantization} (Section~\ref{sec:ivfamp}): the IVF decomposition reduces the squared error of the inner-product estimate by a factor of $1 - 2\E[\ip{x}{c_l}] + \E[\norm{c_l}^2]$, explaining a $+17.7$pp recall jump from flat TQ to IVF-TQ at matched bit budget on SIFT-1M. \textbf{This amplification is what makes a data-independent residual layer viable at IVF scale at all}: without it, codebook-free Lloyd--Max remains in the $\sim$70\% recall regime where production systems would not adopt it. The amplification also closes the residual-variance gap to Extended RaBitQ to within statistical noise (Table~\ref{tab:ivftq_vs_rabitq}).
    \item \textbf{Adaptive coarse refresh} (Sections~\ref{sec:adaptive},~\ref{sec:recovery}): a partition-only refresh that PQ-family indexes structurally cannot perform, enabled by data-independent compression. Under worst-case rotation shift, IVF-TQ frozen drops to 67\% recall. At matched bit budget (no re-ranking), Adaptive IVF-TQ at 4-bit reaches 90.3\% versus PQ retrain's 95.8\% at $\sim$8 bits/dim per subspace ($m{=}96$)---a 5.5pp deficit attributable to the bit-budget gap, not the algorithm. With re-ranking on top-50 candidates, Adaptive IVF-TQ reaches \textbf{97.8\%}, exceeding PQ retrain by 2pp at comparable refresh compute. Adaptive's contribution is therefore the rerank-enabled regime; we report both numbers up front rather than burying the comparison.
\end{enumerate}

\textbf{Empirical headline.} IVF-TQ delivers streaming stability at one fixed $(b, d)$ configuration on all three 10M datasets we test, with $\Delta \in [-0.80, +0.56]$pp (range across all 9 cells at the 10M state; per-cell CIs $\pm 0.4$pp). Across nine controlled cells (three 10M datasets $\times$ three PQ memory budgets, three seeds each), IVF-PQ does not: its streaming stability is \emph{dataset-dependent} in the bit budget. At bit-matched memory ($\sim$0.95$\times$ IVF-TQ), IVF-PQ is stable on Deep-10M ($\Delta=-0.84 \pm 0.26$pp) and on T2I-10M ($\Delta=-0.89 \pm 0.32$pp) but still degrades on SIFT-10M ($\Delta=-2.31 \pm 0.42$pp, gap $+6.79 \pm 0.25$pp behind IVF-TQ); only at super-matched ($\sim$1.5$\times$) memory is IVF-PQ stable on all three datasets.

Per-batch PQ codebook retraining never recovers the streaming gap (full statistics in Tables~\ref{tab:streaming_deep10m_all}--\ref{tab:streaming_t2i10m_all}; \S\ref{sec:streaming_scale}). Under encoder-model swap (\S\ref{sec:embedswap}), 3-seed paired-$t$ on the harshest swap (L6$\to$BGE-small) shows IVF-TQ frozen exceeding actively-retrained IVF-PQ by $+14.72 \pm 0.43$pp ($p<0.001$) at the +300K compute-budgeted state, with a similar $+13.27 \pm 0.77$pp gap on the gentle L6$\to$L12 swap.

IVF-TQ delivers \textbf{streaming stability without per-dataset tuning} at one fixed $(b, d)$ compression configuration (the IVF coarse partition $L$ is still tuned per dataset, as is standard for IVF). This is consistent with the structural data-independence guarantee of Theorem~\ref{thm:ip-uniform-main} (\S\ref{sec:ivfamp}). To match this with IVF-PQ, a practitioner would need to pick $m$ and $b$ per dataset based on the underlying data complexity; IVF-TQ requires no such calibration. Static-recall comparisons against OPQ, ScaNN, HNSW~\cite{malkov2020hnsw}, and Extended RaBitQ are in \S\ref{sec:million_scale}; the streaming-stability claim above is specifically against \emph{learned-codebook} methods, while HNSW and RaBitQ are also codebook-free and compete on different axes.

\section{Background}

\subsection{IVF-PQ and learned quantization}
PQ~\cite{jegou2011pq} partitions a $d$-dimensional vector into $m$ subspaces and learns a $K{=}256$-centroid codebook per subspace via $k$-means. IVF-PQ---the IVFADC scheme of the same paper~\cite{jegou2011pq}---adds a coarse $L$-cell $k$-means partition and compresses only the residual $r = x - c_l$ with PQ. OPQ~\cite{ge2013opq} adds a learned orthogonal rotation that aligns subspaces with the data covariance; ScaNN~\cite{guo2020scann} learns the codebook to minimize an anisotropic loss tailored to inner-product ranking. All three methods produce a codebook that depends on the training sample.

\subsection{TurboQuant: data-independent scalar quantization}
TurboQuant~\cite{zandieh2025turboquant} applies a fixed random orthogonal rotation $\Pi \in \R^{d \times d}$ before quantization. Because $\Pi$ preserves inner products and rotates each coordinate to be approximately $\mathcal{N}(0, 1/d)$ on the unit sphere, a precomputed Lloyd--Max~\cite{lloyd1982least} scalar quantizer of $b$ bits/coord achieves the per-coordinate minimum-MSE rate \emph{independent of the data}. The quantizer never needs re-training. TurboQuant's original Stage~2 used the QJL transform~\cite{zandieh2024qjl} for an unbiased correction; we show this is suboptimal for ranking.

\section{IVF-TQ}
\label{sec:method}

\subsection{Architecture}
\textbf{Indexing.} All vectors live on the unit sphere. At training time, the database is $L^2$-normalised, $k$-means is run on the unit-sphere data, and the resulting centroids are themselves $L^2$-normalised so that every centroid $\tilde c_l \in S^{d-1}$. The IVF coarse partition is the only trained component. At indexing, each input $x$ is $L^2$-normalised to $\tilde x := x/\|x\|_2$ and assigned to its nearest centroid $\tilde c_l$ by maximum inner product. The residual $r := \tilde x - \tilde c_l$ is rotated by $\Pi$ to $\Pi r$, and $\Pi r$ is $L^2$-normalised to a unit vector $\hat\rho := \Pi r / \|\Pi r\|_2$; $\hat\rho$ is quantized per-coordinate by the precomputed Lloyd--Max codebook designed for $\mathcal{N}(0, 1/d)$ (the marginal of a Haar-rotated unit vector; Lemma~\ref{lem:marginal}). Because $\Pi$ is orthogonal, $\|\Pi r\|_2 = \|r\|_2$, so the stored residual norm $\|r\|_2 = \|\tilde x - \tilde c_l\|_2$ serves at search time as the scale factor for asymmetric inner-product reconstruction.

\textbf{Bit accounting.} Stored per vector: $b\!\cdot\!d$ bits for Lloyd--Max codes, $d$ bits for sign-bit refinement, $\lceil \log_2 L\rceil$ bits for the centroid id, and $16$ bits for $\|r\|_2$ (float16). The centroid id and $\|r\|_2$ together contribute a fixed $\approx 32$~bits/vec of \emph{non-residual} overhead, the source of the $+32$ offset in Tables~\ref{tab:streaming_deep10m_all}--\ref{tab:streaming_t2i10m_all}. Worked examples: Deep-10M at $b{=}4$ ($d{=}96$): $4\cdot 96 + 96 + 32 = 512$ bits/vec; SIFT-10M at $b{=}4$ ($d{=}128$): $4\cdot 128 + 128 + 32 = 672$ bits/vec; T2I-10M at $b{=}4$ ($d{=}200$): $4\cdot 200 + 200 + 32 = 1032$ bits/vec.

\textbf{Search.} At query time, $q$ is $L^2$-normalised to $\tilde q$. For each of the $n_p$ probed partitions, the coarse term $\ip{\tilde q}{\tilde c_l}$ is computed exactly, and the residual estimate $\|r\|_2 \cdot \ip{\Pi\tilde q}{\hat\rho}$ is added using the stored residual norm and the quantized unit residual code. Re-ranking on raw vectors is optional; new vectors enter at any time at full compression quality. The Lloyd--Max codebook depends only on $(b, d)$ because the residual between any two unit vectors, after rotation and renormalisation, has the same per-coordinate marginal distribution regardless of the data --- the data-independence property the streaming claim hinges on. Stage~2 sign-bit refinement adds one bit per coordinate; under IVF amplification it provides no measurable benefit (Table~\ref{tab:ivftq_vs_rabitq}; contrast with the flat-TQ regime in Appendix~\ref{app:signbit}, Table~\ref{tab:signbit}), so we treat it as architectural detail rather than as a contribution. Figure~\ref{fig:architecture} summarises the architecture. A two-pass cascade observation is documented in Appendix~\ref{app:cascade}.

\begin{figure*}[!t]
\centering
\resizebox{\textwidth}{!}{%
\begin{tikzpicture}[
  font=\sffamily\small,
  >=Latex,
  node distance=0.6cm and 0.75cm,
  data/.style={draw=black!60, line width=0.5pt, rounded corners=2pt,
               fill=gray!8, minimum width=1.7cm, minimum height=0.85cm,
               align=center, inner sep=3pt, font=\sffamily\footnotesize},
  trained/.style={draw=blue!50!black, line width=0.7pt, rounded corners=4pt,
                  fill=blue!18, minimum width=2.6cm, minimum height=1.05cm,
                  align=center, inner sep=3pt, font=\sffamily\footnotesize,
                  drop shadow={shadow xshift=0.5pt, shadow yshift=-0.5pt, opacity=0.25}},
  free/.style={draw=orange!75!black, line width=0.7pt, rounded corners=4pt,
               fill=orange!22, minimum width=2.6cm, minimum height=1.05cm,
               align=center, inner sep=3pt, font=\sffamily\footnotesize,
               drop shadow={shadow xshift=0.5pt, shadow yshift=-0.5pt, opacity=0.25}},
  arrow/.style={->, line width=0.7pt, draw=black!75}
]
\node[font=\sffamily\bfseries] (lbl_idx) at (-0.7, 0) {\textsc{Index}};
\node[data, right=0.4cm of lbl_idx] (x) {$x \in \mathbb{R}^d$};
\node[trained, right=of x] (km) {Coarse $k$-means \\[1pt] \itshape\scriptsize(trained, $L$ cells)};
\node[data, right=of km] (cid) {centroid id $\ell$};
\node[data, below=0.4cm of cid] (res) {$r = x - c_\ell$};
\node[free, right=of res] (rot) {$\Pi r$ \\[1pt] \itshape\scriptsize fixed random};
\node[free, right=of rot] (lm) {Lloyd--Max \\[1pt] \itshape\scriptsize $b$ bits/coord};
\node[free, right=of lm] (sb) {Sign-bit \\[1pt] \itshape\scriptsize $+1$ bit};
\node[data, right=of sb] (codes) {codes};

\draw[arrow] (x) -- (km);
\draw[arrow] (km) -- (cid);
\draw[arrow] (km.south) |- (res);
\draw[arrow] (res) -- (rot);
\draw[arrow] (rot) -- (lm);
\draw[arrow] (lm) -- (sb);
\draw[arrow] (sb) -- (codes);

\node[font=\sffamily\bfseries] (lbl_s) at (-0.7, -3.3) {\textsc{Search}};
\node[data, right=0.4cm of lbl_s] (q) {$q \in \mathbb{R}^d$};
\node[trained, right=of q] (coarse) {$\langle q, c_\ell\rangle$ \\[1pt] \itshape\scriptsize for all $\ell$};
\node[data, right=of coarse] (top_np) {top-$n_p$ \\ partitions};
\node[free, right=of top_np] (Pq) {$\Pi q$};
\node[free, right=of Pq] (adc) {$\langle \Pi q,\, \hat r_\ell\rangle$ \\[1pt] \itshape\scriptsize asymmetric IP};
\node[data, right=of adc] (rankout) {top-$k$};

\draw[arrow] (q) -- (coarse);
\draw[arrow] (coarse) -- (top_np);
\draw[arrow] (top_np) -- (Pq);
\draw[arrow] (Pq) -- (adc);
\draw[arrow] (adc) -- (rankout);

\draw[->, line width=0.6pt, dashed, draw=black!55]
  (codes.south) to[out=-90, in=90, looseness=1.0] (adc.north);

\end{tikzpicture}%
}
\caption{IVF-TQ architecture. The coarse $k$-means partition (blue) is the
only trained component. The residual compression layer (orange) is
\emph{data-independent}: a fixed random rotation $\Pi$, a precomputed
Lloyd--Max scalar quantizer parameterised only by $(b, d)$, and a sign-bit
refinement of one extra bit per coordinate. At search time, asymmetric
inner-product is computed in the rotated frame; the coarse term
$\langle q, c_\ell\rangle$ is exact, only the residual is compressed.
Dashed arrow: compressed codes from indexing feed the asymmetric IP step. Both inputs and coarse centroids are $L^2$-normalised (\S\ref{sec:method}); the residual $r = \tilde x - \tilde c_\ell$ is between two unit vectors. ``$\Pi r$'' in the figure is shorthand for the rotated residual, which is itself $L^2$-normalised before quantization. $\|r\|_2$ is stored separately.}
\label{fig:architecture}
\end{figure*}
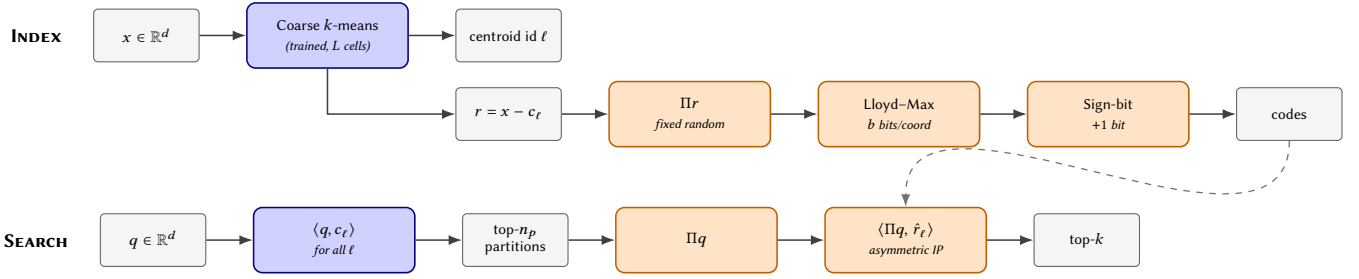

\subsection{IVF amplifies codebook-free quantization}
\label{sec:ivfamp}

The IVF decomposition $\ip{q}{x} = \ip{q}{c_l} + \ip{q}{r}$ keeps the coarse term exact; only the residual $r$ is compressed. Let $\hat{r}$ denote TQ's reconstruction of $r = \tilde x - \tilde c_l$ (unit vectors per \S\ref{sec:method}) and $\epsilon = r - \hat{r}$. Under flat TQ the score-estimation error is $\ip{\tilde q}{\epsilon_x}$ with $\epsilon_x = \tilde x - \hat{\tilde x}$; under IVF-TQ it is $\ip{\tilde q}{\epsilon}$. Because both $\tilde x$ and $\tilde c_l$ are unit vectors and TQ's per-coordinate MSE scales with signal variance, the variance reduction ratio is $\E[\ip{\tilde q}{\epsilon}^2] / \E[\ip{\tilde q}{\epsilon_x}^2] = \E[\norm{r}^2] / \E[\norm{\tilde x}^2] = 2(1 - \E[\ip{\tilde x}{\tilde c_l}])$. For high cluster affinity this ratio is small. On SIFT-1M with $L{=}1000$, $\E[\ip{\tilde x}{\tilde c_l}] \approx 0.85$, giving $2(1 - 0.85) = 0.30$, a $\sim$3.3$\times$ MSE reduction. The empirical jump from flat TQ (69.8\%; Appendix~\ref{app:signbit}, Table~\ref{tab:signbit}) to IVF-TQ at $n_p{=}20$ (87.5\%; Table~\ref{tab:1m}) on SIFT-1M at matched bit budget ($b{=}4$ + sign-bit) is $+17.7$pp, consistent with this prediction.

\textbf{IVF amplification absorbs the marginal benefit of Stage~2.} Table~\ref{tab:ivftq_vs_rabitq} reports the IVF-TQ regime: at matched total bit budget, sign-bit refinement and Extended RaBitQ-equivalent (TQ Stage 1 only at the same total bits) tie within $\pm 0.5$pp across 9 (dataset $\times$ bit-budget) cells, with mean $\Delta = +0.08$pp (paired $t$-test $p{=}0.50$). The $\sim$3$\times$ residual-variance reduction from IVF leaves little headroom for any per-bit refinement to add. This is the empirical justification for treating IVF amplification (not Stage~2) as the operationally significant result; flat-TQ-only Stage~2 results are in Appendix~\ref{app:signbit}.

\begin{table*}[!t]
\centering
\small
\caption{\textbf{IVF-TQ regime}: sign-bit refinement vs.\ Extended RaBitQ-equivalent~\cite{gao2025extendedrabitq} at \emph{matched total memory}. ``Ext.\ RaBitQ ($B$)'' is per-coordinate Lloyd--Max at $B$ bits with no sub-bin refinement (TQ Stage~1 only at $B$ bits, equivalent to Extended RaBitQ at $B$ for $d \geq 64$ up to the $O(1/\sqrt{d})$ marginal-correction term in Theorem~1). ``IVF-TQ'' uses $b{=}B{-}1$ bits Stage~1 + 1 bit Stage~2 (sign-bit), so total bits $=B$ in both columns. Both methods wrapped in IVF ($L{=}1000$, $n_p{=}20$, $n_q{=}1000$, seed~42). Mean $\Delta$ across 9 cells: $+0.08$pp; paired $t$-test $p{=}0.50$; 95\% CI $[-0.19, +0.35]$pp---no detectable difference. Ext.\ RaBitQ rows from our Python reimplementation, which matches the official implementation's~\cite{gao2025extendedrabitq} per-coordinate quantizer up to the $O(1/\sqrt{d})$ Gaussian-marginal correction proven in Theorem~\ref{thm:rd-fixed-main}; we therefore expect no numerical difference at $d \geq 64$.}
\label{tab:ivftq_vs_rabitq}
\begin{tabular}{lccccc}
\toprule
\textbf{Dataset} & \textbf{Total bits} & \textbf{Memory} & \textbf{Ext.\ RaBitQ} & \textbf{IVF-TQ (ours)} & $\Delta$ \\
\midrule
\multirow{3}{*}{SIFT-1M} & 4 &  65 MB & 80.14\% & 79.76\% & $-0.38$pp \\
                        & 5 &  81 MB & 87.77\% & 87.78\% & $+0.01$pp \\
                        & 6 &  96 MB & 91.07\% & 91.59\% & $+0.52$pp \\
\midrule
\multirow{3}{*}{Deep-1M} & 4 & 50 MB & 83.57\% & 83.75\% & $+0.18$pp \\
                         & 5 & 61 MB & 89.74\% & 89.93\% & $+0.19$pp \\
                         & 6 & 73 MB & 92.47\% & 92.80\% & $+0.33$pp \\
\midrule
\multirow{3}{*}{GloVe-1M~\cite{pennington2014glove}} & 4 & 52 MB & 76.71\% & 76.21\% & $-0.50$pp \\
                          & 5 & 64 MB & 81.07\% & 81.03\% & $-0.04$pp \\
                          & 6 & 76 MB & 82.92\% & 83.35\% & $+0.43$pp \\
\midrule
\multicolumn{5}{l}{\textbf{Mean $\Delta$ across 9 cells}} & \textbf{$+0.08$pp} \\
\multicolumn{5}{l}{\textbf{Range}}                        & $[-0.50, +0.52]$pp \\
\bottomrule
\end{tabular}
\end{table*}

\textbf{Recall ceiling of learned-codebook indexes.} For any fixed PQ configuration, IVF-PQ exhibits a per-configuration recall ceiling that scales with the bit budget; full numbers across $m \in \{8, 16, 32, 64, 128\}$ are in Appendix~\ref{app:pq-ceiling}. IVF-TQ's recall ceiling at fixed bits/dim is set by $(b, d)$ alone via the rate-distortion floor $\sqrt{D_b}$ on residual reconstruction error (Theorem~\ref{thm:rd-fixed-main}); recall scales smoothly with $n_p$ up to that ceiling.

\textbf{Rate-distortion bound for TQ residual quantization.} The TQ
reconstruction $\hat{v} := \Pi^\top C_b(\Pi v)$ admits a high-probability
error bound depending only on $(b, d, \delta)$. Let
$D_b := d \cdot \E_{T \sim \mathcal{N}(0, 1/d)}[(C_b(T) - T)^2]$ denote
the $b$-bit rate-distortion limit for the per-coordinate Gaussian source.
\cite{zandieh2025turboquant} prove an in-expectation MSE bound for this
quantizer (their Theorem~1); we state the standard high-probability
counterpart below as proof scaffolding for our uniform IP-error bound
(Theorem~\ref{thm:ip-uniform-main}). The expectation-to-high-probability
lift is a routine L\'evy--Milman concentration argument and is \emph{not
claimed as a new result}; we restate it in our notation for
self-containedness.

\begin{theorem}[High-probability MSE bound for TQ; restatement of \cite{zandieh2025turboquant} Thm.~1 in high-probability form; proof in Appendix~\ref{app:concentration}]
\label{thm:rd-fixed-main}
Fix any unit vector $v \in S^{d-1}$. For any $\delta \in (0,1)$, with
probability at least $1 - \delta$ over the random rotation $\Pi$:
\begin{equation}
    \|\hat{v} - v\|_2 \;\leq\; \sqrt{D_b} + R_d + \sqrt{\frac{8 \log(2/\delta)}{d-2}},
\end{equation}
where $R_d = O(1/\sqrt{d})$ captures the deviation of the marginal of
$(\Pi v)_j$ from the Gaussian source for which $C_b$ is designed.
\end{theorem}

Theorem~\ref{thm:rd-fixed-main} bounds the reconstruction error of a
single fixed vector. \cite{zandieh2025turboquant} additionally prove a
per-pair, in-expectation bound on the inner-product estimation error
(their Theorem~2). For streaming ANN we need a stronger statement: a
uniform bound over the entire unit sphere with one fixed rotation.
The union bound is then consumed once at index initialisation, over an
$\epsilon$-net of the sphere whose log-size depends on $d$ but not on
$N$. Adding the $N$-th database vector consumes no additional budget.
We provide such a bound below; it is the genuine theoretical
contribution of this paper.

\begin{theorem}[Uniform-over-sphere IP-error bound for TQ; novel; proof in Appendix~\ref{app:concentration}]
\label{thm:ip-uniform-main}
Under the setup of Theorem~\ref{thm:rd-fixed-main}, with probability $\geq 1-\delta$ over $\Pi$, the inner-product error $|\langle q, v\rangle - \langle q, \hat{v}\rangle|$ is bounded uniformly over all $v \in S^{d-1}$ by a quantity depending only on $(b, d, \delta)$; the explicit form and proof are in Appendix~\ref{app:concentration}.
\end{theorem}

\textbf{Tightness.} Theorem~\ref{thm:ip-uniform-main} is a structural statement, not a numerically predictive one. It establishes two properties --- \emph{data-independence} (the bound depends only on $(b, d, \delta)$ and has no term that grows with $N$ or with distribution drift) and \emph{uniformity over $S^{d-1}$ with one fixed $\Pi$} (the union bound is consumed once at index initialisation over an $\epsilon$-net of the sphere, so adding the $N$-th database vector incurs no additional budget). No learned-codebook PQ-family method admits an analogous uniform, data-independent bound: their reconstruction error depends on the distance from $v$ to the nearest learned codebook centroid, a function of the training sample.\footnote{RaBitQ admits its own per-vector data-independent bound; the contrast above is specifically with the PQ family of learned-codebook methods.} The asymptotic advantage over Cauchy--Schwarz on $\|v - \hat v\|_2$ materialises at $d \gtrsim 10^3$; at our $d{=}128$ operating regime the bound is numerically loose (end-to-end numerical comparison in Appendix~\ref{app:concentration}; the term-by-term breakdown and the $d{=}10^3$/$10^4$ extrapolation are in Appendix~\ref{app:t2-numerics}). The empirical streaming results (\S\ref{sec:streaming}) are consistent with the structural picture but are not predicted by the numerics.

\subsection{Adaptive coarse refresh}
\label{sec:adaptive}

The bound in the previous subsection covers only the residual quantization layer; the IVF coarse partition is still fitted by $k$-means and can become stale under severe distribution shift (e.g.\ encoder swap). We propose \emph{Adaptive IVF-TQ}: periodically re-run $k$-means on a stratified sample of the currently indexed vectors and re-encode every vector against the new partition. Because TQ residual compression is data-independent, re-encoding requires no codebook re-training---only an $O(N \cdot d)$ rotate-and-quantize pass with the same fixed Lloyd--Max codebook used at index initialization.

This selective refresh property is structurally only possible with TQ. PQ and OPQ couple the partition and the compression layer through the codebook; refreshing one forces re-training the other. Adaptive IVF-TQ refreshes only what depends on the data ($k$-means centroids), which we show is sufficient to recover from severe distribution shift at refresh cost comparable to a single PQ codebook re-training (50s vs.\ 39s in Table~\ref{tab:recovery}), while attaining higher final recall when re-ranking is enabled.

\textbf{Algorithm.} On every refresh trigger (e.g.\ every $N_{\text{refresh}}$ ingested vectors), Adaptive IVF-TQ samples $S$ vectors stratified across partitions, re-clusters to obtain new centroids, re-assigns every indexed vector using its TQ-reconstructed approximation (or its raw vector if stored for re-rank), re-rotates and re-quantizes the new residuals using the unchanged Lloyd--Max codebook, and discards old centroids. The compression layer parameters $(\Pi, C_b)$ never change.

\textbf{Refresh self-containedness.} The re-assignment step does not require raw-vector storage. The TQ-reconstructed approximation $\hat{\tilde x} = \tilde c_{\text{old}} + \|r\|_2 \cdot \Pi^\top \hat\rho$ (where $\hat\rho$ is the quantized unit residual in the rotated frame; \S\ref{sec:method}) has error bounded by Theorem~\ref{thm:rd-fixed-main}; re-clustering on $\hat{\tilde x}$ converges to centroids close to those obtained from raw vectors in our experiments.

\section{Streaming Experiments}
\label{sec:streaming}

The paper uses two distinct experimental matrices that share notation:
a \textbf{9-cell static-recall matrix} (3 datasets $\times$ 3 bit budgets) underlying Table~\ref{tab:ivftq_vs_rabitq} in \S\ref{sec:ivfamp}, and a \textbf{9-cell streaming matrix} (3 datasets $\times$ 3 PQ memory regimes, 3 seeds each) underlying Tables~\ref{tab:streaming_deep10m_all}--\ref{tab:streaming_t2i10m_all} in this section. References to ``$X$ of $N$ cells'' refer to whichever matrix is locally being discussed.

We first establish the \emph{mechanism} of streaming recall degradation
in PQ-family indexes (\S\ref{sec:streaming_mechanism}), then quantify
the gap to IVF-TQ at SIFT-1M and Deep-10M scale (\S\ref{sec:streaming_scale}),
under embedding-model swap (\S\ref{sec:embedswap}), and the cumulative
re-training cost of staying competitive (\S\ref{sec:cost_fresh}).
\S\ref{sec:recovery} stress-tests the IVF coarse-partition layer under
worst-case shift.

\subsection{Mechanism: capacity-bound, dataset-dependent}
\label{sec:streaming_mechanism}

A learned PQ codebook is fitted to a sample of $N_0$ vectors and held fixed as the database grows to $N \gg N_0$. Across three memory budgets ($\sim$0.75$\times$, $\sim$0.95$\times$, $\sim$1.5$\times$ IVF-TQ memory) and three 10M datasets (Deep-10M, SIFT-10M, T2I-10M), we find that streaming PQ recall behavior is governed by bit budget in a dataset-dependent way --- neither codebook freshness nor distribution shift explains the observed degradation pattern. The cross-regime evidence is in Tables~\ref{tab:streaming_deep10m_all}--\ref{tab:streaming_t2i10m_all}.

First, \textbf{distribution shift alone is not sufficient.} Under shuffled-i.i.d.\ ingestion on SIFT-1M (no distribution shift), IVF-PQ at sub-matched memory still degrades $-3.94 \pm 0.17$pp --- comparable to the $-4.2$pp under original order --- ruling out distribution shift as a sufficient explanation for the streaming gap. Full three-condition control table is in Appendix~\ref{app:streaming-controls}.

Second, \textbf{codebook freshness alone is not sufficient.} Across 9 cells (3 datasets $\times$ 3 memory regimes), per-batch retraining + re-encoding never recovers the streaming gap: in 8 of 9 cells it is statistically indistinguishable from no retraining (paired-$t$ $p \geq 0.14$), and in the 9th cell (SIFT-10M at bit-matched memory) the paired difference is retraining $-$ stale $= -0.08$pp ($p{=}0.040$) --- statistically significant but practically negligible, and in the direction opposite to a recovery effect (Tables~\ref{tab:streaming_deep10m_all}--\ref{tab:streaming_t2i10m_all}). At sub-matched memory all three datasets show large PQ degradation; at super-matched ($\sim$1.5$\times$ IVF-TQ) all three stabilize; at bit-matched ($\sim$0.95$\times$ IVF-TQ) the behavior splits --- Deep-10M and T2I-10M stabilize ($\Delta=-0.84 \pm 0.26$pp and $\Delta=-0.89 \pm 0.32$pp respectively at 10M) while SIFT-10M still degrades ($\Delta=-2.31 \pm 0.42$pp, gap $+6.79 \pm 0.25$pp behind IVF-TQ). The capacity threshold differs across the three datasets in our experiments; we observe this empirically but do not isolate the specific data property (per-coordinate variance, cluster structure, or intrinsic dimensionality) that determines the threshold. Codebook freshness, however, is ruled out at every memory regime tested.

IVF-TQ's residual layer is governed only by $(b, d)$, not by data, so its streaming recall behavior is invariant across all three datasets at a single fixed configuration ($\Delta \in [-0.80, +0.56]$pp, range across all three 10M datasets at the single fixed $(b, d)$ configuration; per-cell CIs $\pm 0.4$pp). To match this with IVF-PQ, a practitioner would need to pick $m$ and $b$ per dataset based on the underlying data complexity --- and re-validate after every distribution change. IVF-TQ requires no such calibration: streaming stability without per-dataset tuning, at one fixed $(b, d)$ configuration --- a property we demonstrate empirically (\S\S\ref{sec:streaming_scale}--\ref{sec:recovery}) and that is consistent with the structural data-independence guarantee of Theorem~\ref{thm:ip-uniform-main} (\S\ref{sec:ivfamp}).

\textbf{Capacity vs.\ bias control.} A natural follow-up is whether the streaming drop is a codebook-bias artefact (the initial 200K is unrepresentative) or a codebook-capacity artefact (a 200K-trained codebook is too small for a 1M database). We disambiguate by training three IVF-PQ variants at matched protocol (200K-initial, 200K-random, 1M-oracle) and evaluating each against the same 1M database; both the bias contribution (B$-$A $= -0.25$pp) and the residual capacity contribution (C$-$B $= -0.15$pp) are within statistical noise (SE $\approx 1.4$pp). At this bit budget, which 200K is used to train the codebook does not detectably matter --- jointly with the 10M retrain experiments below, the negative result is that retraining the codebook is not the lever that closes the streaming gap. Full table in Appendix~\ref{app:capacity-vs-bias}.

In contrast, IVF-TQ's compression layer is parameterised by $(b, d)$
alone (Theorem~\ref{thm:rd-fixed-main}); its per-vector reconstruction
error has the same high-probability bound for the first vector and the
$N$-th, regardless of which distribution they came from. Recall
\emph{improves} under streaming ingestion because partition coverage,
the only data-dependent layer, grows. Distribution shift, when it does
occur, makes the PQ situation worse (\S\ref{sec:embedswap}, an L6$\to$BGE
encoder swap collapses stale IVF-PQ to 51\%) but is not required to
trigger streaming degradation. The streaming problem is therefore better
described as: at sub-matched and dataset-sensitive near-matched bit budgets, the learned codebook is the binding constraint; the data-independent residual layer of IVF-TQ replaces this binding constraint with a fixed $(b, d)$-only error model.

\subsection{Streaming ingestion at scale}
\label{sec:streaming_scale}
\textbf{Protocol.} Train all indexes on the first $N_0$ vectors, then add the remainder in batches. Re-compute Recall@10 on a fixed query set with ground truth recomputed against the cumulative database. \textbf{The retrain variant re-trains a fresh PQ codebook on the cumulative data \emph{and re-encodes every previously-added vector}} (no stale codes survive), which justifies the cumulative compute numbers below. IVF-TQ search uses no re-ranking (raw vectors freed after add), so its recall numbers are directly comparable to the no-rerank IVF-PQ baselines. Coarse-partition parameters ($n_{\mathrm{list}}$, $n_p$) are constant across batches.

\begin{figure*}[!t]
\centering
\includegraphics[width=0.85\textwidth]{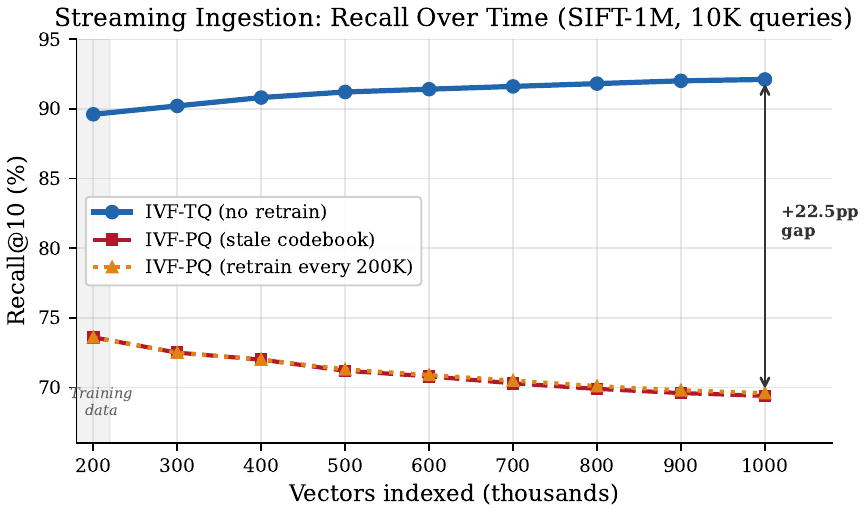}
\caption{Streaming ingestion on SIFT-1M (200K trained, then 8 batches of 100K). IVF-TQ recall \emph{improves} as partition coverage grows; IVF-PQ recall \emph{degrades} as the codebook becomes increasingly mismatched to the cumulative data. Periodic re-training every 200K (15.4s total) recovers +0.2pp.}
\label{fig:streaming}
\end{figure*}

Figure~\ref{fig:streaming} shows the SIFT-1M result: a $+2.51 \pm 0.09$pp improvement for IVF-TQ vs.\ a $-3.94 \pm 0.17$pp degradation for IVF-PQ (3 seeds, both $p<0.001$), with periodic re-training closing only $+0.2$pp of the gap at 15.4s of cumulative training cost. The gap widens monotonically: 16.3pp at 200K, 22.8pp at 1M.

\textbf{Scaling to 10M --- three memory regimes.} We extend to three 10M datasets: Deep-10M (96-dim, ResNet image features), SIFT-10M (128-dim, the first 10M of SIFT-1B), and T2I-10M (200-dim CLIP-style, the first 10M of Yandex Text2Image-1B). For each, we train on the first 1M vectors and stream in 9 batches of 1M each (3 seeds: 42, 123, 7777; full per-batch trajectories in Appendix~\ref{app:streaming-trajectories}). We characterize three PQ memory regimes per dataset: \emph{sub-matched} ($\sim$0.75--0.78$\times$ IVF-TQ memory), \emph{bit-matched} ($\sim$0.95--0.97$\times$; the closest standard-FAISS configuration to IVF-TQ's bit budget), and \emph{super-matched} ($\sim$1.5--1.55$\times$).\footnote{FAISS PQ requires $d \% m_{\text{PQ}} = 0$, which constrains the available regimes per dataset. Deep-10M uses $m \in \{48, 96\}$ (factors of 96); SIFT-10M uses $m \in \{64, 128\}$ (factors of 128); T2I-10M uses $m \in \{100, 200\}$ (factors of 200). The standard sub-matched $m{=}48$ regime is not available at $d{=}200$, so T2I uses $m{=}100$ at $b{=}8$ (800 bits, $\sim$0.78$\times$).} Tables~\ref{tab:streaming_deep10m_all}--\ref{tab:streaming_t2i10m_all} report all nine cells; the headline finding is that PQ's streaming behavior splits across datasets at bit-matched memory.

\begin{table*}[!t]
\centering\small
\caption{\textbf{Deep-10M, bit-matched memory regime ($\sim$0.95$\times$ IVF-TQ): IVF-PQ stabilises here but the threshold differs across datasets (Tables~\ref{tab:streaming_sift10m_all},~\ref{tab:streaming_t2i10m_all}).} Full three-regime tables (sub-/bit-/super-matched) in Appendix~\ref{app:streaming-full}. 3 seeds (42, 123, 7777); mean $\pm$ 95\% CI; paired-$t$ on within-seed differences. IVF-TQ at $b{=}4$ + sign-bit refinement (512 bits/vec).}
\label{tab:streaming_deep10m_all}
\begin{tabular}{llcccc}
\toprule
Regime & Index & Bits/vec & R@10 (1M) & R@10 (10M) & Change $\Delta$ \\
\midrule
\multicolumn{2}{l}{\emph{IVF-TQ (single configuration)}} & 512 & $87.39{\pm}0.33$ & $86.59{\pm}0.17$ & $-0.80{\pm}0.25$ pp, $p{=}0.005$ \\
\midrule
\multirow{3}{*}{\shortstack[l]{Bit-matched\\($\sim$0.95$\times$ IVF-TQ;\\$m{=}48$, $b{=}10$)}}
  & IVF-PQ stale     & 480 & $87.21{\pm}0.28$ & $86.37{\pm}0.21$ & $-0.84{\pm}0.26$ pp, $p{=}0.005$ \\
  & IVF-PQ retrain   & 480 & $87.21{\pm}0.28$ & $86.41{\pm}0.23$ & $-0.80{\pm}0.38$ pp, $p{=}0.012$ \\
  & \multicolumn{5}{l}{\footnotesize\itshape IVF-TQ vs.\ PQ stale at 10M: $+0.22{\pm}0.05$pp ($p{=}0.003$); retrain $-$ stale $+0.04{\pm}0.43$pp ($p{=}0.737$)} \\
\bottomrule
\end{tabular}
\end{table*}

\begin{table*}[!t]
\centering\small
\caption{\textbf{SIFT-10M, bit-matched memory regime ($\sim$0.95$\times$ IVF-TQ): IVF-PQ still degrades $-2.31$pp here, while it stabilises on Deep-10M (Table~\ref{tab:streaming_deep10m_all}) and on T2I-10M (Table~\ref{tab:streaming_t2i10m_all}) --- the capacity threshold for PQ streaming stability is dataset-dependent.} Full three-regime tables in Appendix~\ref{app:streaming-full}. 3 seeds; mean $\pm$ 95\% CI. IVF-TQ at $b{=}4$ + sign-bit refinement (672 bits/vec).}
\label{tab:streaming_sift10m_all}
\begin{tabular}{llcccc}
\toprule
Regime & Index & Bits/vec & R@10 (1M) & R@10 (10M) & Change $\Delta$ \\
\midrule
\multicolumn{2}{l}{\emph{IVF-TQ (single configuration)}} & 672 & $83.91{\pm}0.08$ & $84.47{\pm}0.28$ & $+0.56{\pm}0.10$ pp, $p{=}0.007$ \\
\midrule
\multirow{3}{*}{\shortstack[l]{Bit-matched\\($\sim$0.95$\times$ IVF-TQ;\\$m{=}64$, $b{=}10$)}}
  & IVF-PQ stale     & 640 & $80.00{\pm}0.44$ & $77.69{\pm}0.04$ & $-2.31{\pm}0.42$ pp, $p{=}0.002$ \\
  & IVF-PQ retrain   & 640 & $80.00{\pm}0.44$ & $77.61{\pm}0.12$ & $-2.40{\pm}0.38$ pp, $p{=}0.001$ \\
  & \multicolumn{5}{l}{\footnotesize\itshape IVF-TQ vs.\ PQ stale at 10M: $+6.79{\pm}0.25$pp ($p{<}0.001$); retrain $-$ stale $-0.08{\pm}0.07$pp ($p{=}0.040$, negligible magnitude)} \\
\bottomrule
\end{tabular}
\end{table*}

\textbf{Three-regime characterization.} At sub-matched memory IVF-PQ degrades on all three datasets; at super-matched ($\sim$1.5$\times$) IVF-PQ stabilises on all three and exceeds IVF-TQ in raw recall as expected from rate-distortion (full tables in Appendix~\ref{app:streaming-full}). The interesting case, reported in Tables~\ref{tab:streaming_deep10m_all}--\ref{tab:streaming_t2i10m_all}, is bit-matched ($\sim$0.95$\times$): \textbf{IVF-PQ stabilises on Deep-10M and T2I-10M but still degrades on SIFT-10M.} The capacity threshold for PQ streaming stability is therefore dataset-dependent and not predictable from bit budget alone.

\textbf{Retraining never recovers.} At the bit-matched (headline) regime, per-batch retraining costs $667 \pm 20$s on Deep-10M, $821 \pm 7$s on SIFT-10M, and $1328 \pm 43$s on T2I-10M of cumulative compute --- one representative cell per dataset --- with zero recall benefit at any tested memory regime (Tables~\ref{tab:streaming_deep10m_all}--\ref{tab:streaming_t2i10m_all}). The recommendation for production systems running IVF-PQ at sub-matched or near-matched capacity is straightforward: \emph{stop retraining the codebook --- it costs compute and does not help}. The only mechanism that closes the streaming recall gap on PQ is adding bits (super-matched), and adding bits has costs of its own.

\textbf{IVF-TQ across all 9 cells.} IVF-TQ's recall trajectory at one fixed $(b{=}4 +$~sign-bit$, d)$ configuration has $\Delta \in [-0.80, +0.56]$pp (range across all 9 cells; per-cell CIs $\pm 0.4$pp). No per-dataset capacity tuning is required for streaming stability; this is the durable contribution at every tested memory budget.

\begin{table*}[!t]
\centering\small
\caption{\textbf{T2I-10M (200-dim CLIP-style; first 10M of Yandex Text2Image-1B), bit-matched memory regime ($\sim$0.97$\times$ IVF-TQ): IVF-PQ stabilises here, joining Deep-10M on the stable side of the bit-matched threshold (Table~\ref{tab:streaming_deep10m_all}).} Full three-regime tables in Appendix~\ref{app:streaming-full}. 3 seeds (42, 123, 7777); mean $\pm$ 95\% CI; paired-$t$ on within-seed differences. IVF-TQ at $b{=}4$ + sign-bit refinement (1032 bits/vec). $m_{\text{PQ}}$ is constrained by $d \% m_{\text{PQ}} = 0$ at $d{=}200$.}
\label{tab:streaming_t2i10m_all}
\begin{tabular}{llcccc}
\toprule
Regime & Index & Bits/vec & R@10 (1M) & R@10 (10M) & Change $\Delta$ \\
\midrule
\multicolumn{2}{l}{\emph{IVF-TQ (single configuration)}} & 1032 & $81.64{\pm}0.33$ & $80.88{\pm}0.10$ & $-0.76{\pm}0.41$ pp, $p{=}0.015$ \\
\midrule
\multirow{3}{*}{\shortstack[l]{Bit-matched\\($\sim$0.97$\times$ IVF-TQ;\\$m{=}100$, $b{=}10$)}}
  & IVF-PQ stale     & 1000 & $81.48{\pm}0.24$ & $80.59{\pm}0.22$ & $-0.89{\pm}0.32$ pp, $p{=}0.007$ \\
  & IVF-PQ retrain   & 1000 & $81.48{\pm}0.24$ & $80.50{\pm}0.06$ & $-0.98{\pm}0.23$ pp, $p{=}0.003$ \\
  & \multicolumn{5}{l}{\footnotesize\itshape IVF-TQ vs.\ PQ stale at 10M: $+0.29{\pm}0.24$pp ($p{=}0.018$); retrain $-$ stale $-0.09{\pm}0.28$pp ($p{=}0.317$)} \\
\bottomrule
\end{tabular}
\end{table*}

\subsection{Distribution shift via embedding model swap}
\label{sec:embedswap}

When production teams upgrade their embedding model, the corpus distribution shifts. We test two encoder swaps on 1M MS~MARCO~\cite{nguyen2016msmarco} passages: a \emph{gentle} swap (\path{all-MiniLM-L6-v2}~\cite{reimers2019sbert} $\to$ \path{all-MiniLM-L12-v2}, cosine $0.51$ on shared passages) and a \emph{harsh} swap (L6 $\to$ \path{BAAI/bge-small-en-v1.5}, cosine $0.24$). Indexes are trained on $200{,}000$ L6-encoded passages and stream in up to $800{,}000$ new-encoder passages in 100K batches, with queries encoded by the new model. Each experiment uses 3 seeds (42, 123, 7777) with paired-$t$ CIs across seeds. Full per-batch encoder-swap trajectories are in Appendix~\ref{app:encoder-swap-full}.

\begin{table}[!htbp]
\centering\small
\caption{Embedding-model swap on 1M MS~MARCO passages (Recall@10, mean $\pm$ 95\% CI across 3 seeds; paired-$t$ on within-seed differences; full per-batch trajectories in Appendix~\ref{app:encoder-swap-full}). Gentle (L6$\to$L12) row is the final +800K state; harsh (L6$\to$BGE) row is the compute-budgeted +300K state.}
\label{tab:embedswap_main}
\resizebox{\columnwidth}{!}{%
\begin{tabular}{lcccccc}
\toprule
\textbf{Swap} & \textbf{cos} & \textbf{IVF-TQ} & \textbf{IVF-PQ stale} & \textbf{IVF-PQ retrain} & \textbf{Retrain cum.} & \textbf{Paired $\Delta$ ($p$)} \\
\midrule
L6 $\to$ L12 (gentle)  & 0.51 & \textbf{88.83$\pm$0.31}\% & 72.31$\pm$0.22\% & 75.56$\pm$0.33\% & 358.0$\pm$21.0s & \textbf{+13.27$\pm$0.77}pp ($p{<}0.001$) \\
L6 $\to$ BGE (harsh)\,* & 0.24 & \textbf{86.96$\pm$0.54}\% & 52.05$\pm$0.66\% & 72.24$\pm$0.43\% & 91.8$\pm$3.2s   & \textbf{+14.72$\pm$0.43}pp ($p{<}0.001$) \\
\bottomrule
\end{tabular}%
}
\end{table}

IVF-TQ frozen exceeds actively-retrained IVF-PQ in both swap regimes (Table~\ref{tab:embedswap_main}). Gentle (L6$\to$L12) at +800K: IVF-TQ $88.83 \pm 0.31$\% vs.\ IVF-PQ retrain $75.56 \pm 0.33$\% after $358.0 \pm 21.0$s of codebook re-training (paired-$t$ gap $+13.27 \pm 0.77$pp, $p<0.001$). Harsh (L6$\to$BGE) at +300K: IVF-TQ $86.96 \pm 0.54$\% vs.\ IVF-PQ retrain $72.24 \pm 0.43$\% after $91.8 \pm 3.2$s ($+14.72 \pm 0.43$pp, $p<0.001$).

The mechanism: data-independent compression absorbs new-encoder vectors at full quality regardless of shift severity, while partition coverage from prior-encoder vectors is sufficient to keep coarse assignments useful.

\subsection{Cost of staying fresh}
\label{sec:cost_fresh}
Figure~\ref{fig:retrain_cost} plots cumulative re-training cost vs.\ recall over the streaming run. For IVF-PQ there is no Pareto-dominant operating point: low re-training frequency leaves recall low; high frequency consumes significant compute and still does not match IVF-TQ. IVF-TQ sits at zero re-training cost on the recall axis dominated by the leftmost point of the IVF-PQ curve.

\subsection{Stress test: Adaptive IVF-TQ under adversarial shift}
\label{sec:recovery}

\S\ref{sec:embedswap} showed (3-seed paired-$t$) that under realistic
distribution shift (encoder upgrade), IVF-TQ frozen absorbs new-encoder
vectors well; on the harshest swap the freshly-retrained PQ baseline
lags by $+14.72 \pm 0.43$pp ($p<0.001$) at the +300K state. To probe the operational limits of the
coarse-partition layer and motivate the Adaptive variant, we apply a
strictly adversarial shift: the streaming portion is rotated by a
random orthogonal matrix that decorrelates space-A and space-B vectors
(cosine similarity drops to $\sim 0$). This is harsher than any
realistic encoder swap and is designed specifically to \emph{force} the
coarse partition to misalign---making it the worst-case scenario in
which Adaptive IVF-TQ should be evaluated.

Table~\ref{tab:recovery} compares four indexes on this adversarial workload: IVF-TQ frozen (no refresh), Adaptive IVF-TQ (refresh every 100K), IVF-PQ stale, and IVF-PQ retrain (codebook re-trained every batch). Hold-out queries are taken from space~B.

\begin{table*}[!t]
\centering
\small
\caption{Recovery from worst-case rotation shift on Deep-1M. Adaptive IVF-TQ refresh fully recovers the partition layer at refresh compute comparable to PQ retrain (50s vs.\ 39s). With re-rank, Adaptive IVF-TQ exceeds PQ retrain by 2pp. PQ retrain does not store raw vectors for re-ranking under the FAISS baseline; the rr=50 column is therefore n/a for that row.}
\label{tab:recovery}
\begin{tabular}{lcccc}
\toprule
\textbf{Method} & \textbf{Initial} & \textbf{Final R@10 (rr=50)} & \textbf{Final R@10 (rr=0)} & \textbf{Refresh cost} \\
\midrule
IVF-TQ frozen        & 67\%   & 75\%   & 67\%   & --- \\
\textbf{Adaptive IVF-TQ}      & \textbf{67\%}   & \textbf{97.8\%} & 90.3\%  & \textbf{50s} \\
IVF-PQ stale         & 68\%   & 64\%   & 64\%   & --- \\
IVF-PQ retrain        & 68\%   & n/a   & 95.8\% & 39s \\
\bottomrule
\end{tabular}
\end{table*}

\textbf{Reading.} At matched bit budget without re-ranking, PQ retrain at $m{=}96$ ($\sim$8~bits/dim per subspace) reaches 95.8\% R@10 versus Adaptive IVF-TQ at 4-bit ($\sim$5~effective bits/dim) at 90.3\% --- a 5.5pp deficit attributable to the bit-budget gap, not the algorithm. With re-ranking on top-50 candidates against raw vectors, Adaptive IVF-TQ reaches 97.8\% versus PQ retrain's 95.8\% (PQ retrain has no re-rank under FAISS, which does not store raw vectors), at comparable refresh compute (50s vs.\ 39s). The Adaptive contribution stands in the rerank-enabled regime, which is the practical deployment for systems that already store raw vectors. Under realistic encoder shifts (\S\ref{sec:embedswap}), IVF-TQ frozen already wins by $+14.72 \pm 0.43$pp; the $+30$pp recovery in Table~\ref{tab:recovery} is a worst-case capability, not a routine result, but it is a capability \emph{uniquely enabled} by data-independent compression, since PQ-family indexes cannot refresh the partition without re-training the coupled codebook. A refresh-frequency ablation (every 25K through 250K vectors) shows the schedule is highly forgiving; the full sweep is available in the released source at \url{https://github.com/tarun-ks/turboquant_search}.

\begin{figure*}[!t]
\centering
\includegraphics[width=0.85\textwidth]{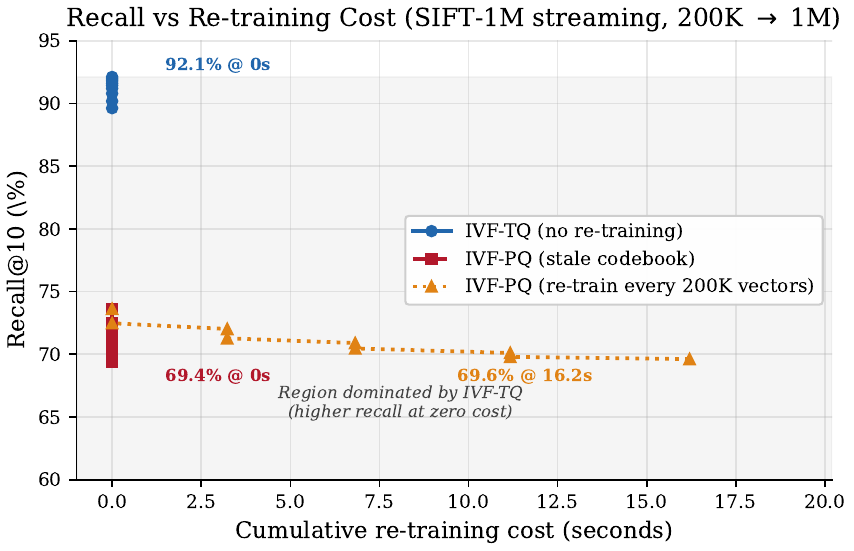}
\caption{Recall@10 vs.\ cumulative codebook re-training cost (SIFT-1M streaming). IVF-TQ requires zero re-training (single point at the origin), while IVF-PQ re-training has no good operating point: low frequency leaves recall stale, high frequency burns compute without matching IVF-TQ.}
\label{fig:retrain_cost}
\end{figure*}

\section{Million-Scale Comparison}
\label{sec:million_scale}

Table~\ref{tab:1m} compares IVF-TQ against the full set of FAISS baselines (PQ, IVF-PQ at $m \in \{64, 128\}$, OPQ+IVF-PQ at $m \in \{64, 128\}$, HNSW at $M \in \{16, 32\}$) and Google's ScaNN~\cite{guo2020scann} (anisotropic learned quantization with reorder, run on Colab Linux; full setup and $L_s$ sweep in Appendix~\ref{app:scann}). All numbers are 10K queries, deterministic seed=42, FAISS 1.13.2.

\textbf{Bit accounting and memory.} ``IVF-TQ $b$-bit'' denotes $b$ bits/coordinate of Lloyd--Max plus the 1-bit sign refinement (Stage~2), giving $b{+}1$ effective bits/coordinate. The Memory column reports the total compressed footprint, which includes the IVF coarse-partition centroids ($L \cdot d \cdot 4$ bytes; $\sim$0.5~MB for $L{=}1000, d{=}128$). $^\dagger$ScaNN's listed memory is the compressed AH+tree footprint matching the IVF-TQ accounting basis; ScaNN additionally stores raw vectors for reorder (550~MB total on SIFT-1M, comparable to HNSW). HNSW stores raw vectors so its memory column reflects that.

\begin{table*}[!t]
\centering
\small
\caption{Recall@10 on million-scale benchmarks (SIFT-1M and Deep-10M; Deep-1M results are in Appendix~\ref{app:million-scale-deep1m} and mirror the SIFT-1M pattern). ScaNN dominates the static-recall axis at fixed compressed memory; IVF-TQ matches OPQ and is competitive with ScaNN below 96\% recall but trails it above. IVF-TQ's differentiation from learned-codebook methods lies in its codebook-free residual layer, which is consequential under streaming workloads (\S\ref{sec:streaming}) rather than on the static-recall axis above. The implementation-latency gap to FAISS C++/SIMD is discussed in~\S\ref{sec:limits}. $^\S$Extended RaBitQ rows: our Python reimplementation, which matches the official implementation's~\cite{gao2025extendedrabitq} per-coordinate quantizer up to the $O(1/\sqrt{d})$ Gaussian-marginal correction proven in Theorem~\ref{thm:rd-fixed-main}; we therefore expect no numerical difference at $d \geq 64$.}
\label{tab:1m}
\begin{tabular}{lcccc}
\toprule
\textbf{Method} & \textbf{R@10} & \textbf{QPS} & \textbf{Memory} & \textbf{Codebook training?} \\
\midrule
\multicolumn{5}{l}{\textit{SIFT-1M (1M, dim=128)}} \\
\midrule
FAISS IVF-PQ m=64, $n_p$=80 & 73.2\% & 7.8K & 62 MB & PQ \\
FAISS IVF-PQ m=128, $n_p$=20 & 89.9\% & 15K & 123 MB & PQ \\
FAISS OPQ+IVF-PQ m=128, $n_p$=20 & 93.4\% & 15K & 123 MB & OPQ + PQ \\
FAISS OPQ+IVF-PQ m=128, $n_p$=80 & 97.0\% & 4.1K & 123 MB & OPQ + PQ \\
FAISS HNSW M=32, $ef_s$=64 & 98.2\% & 76K & 732 MB & None \\
FAISS HNSW M=32, $ef_s$=256 & 99.8\% & 22K & 732 MB & None \\
ScaNN AH+tree, $L_s$=50 & 96.2\% & 5.9K & 62 MB$^\dagger$ & ScaNN AH \\
ScaNN AH+tree, $L_s$=100 & 98.6\% & 3.2K & 62 MB$^\dagger$ & ScaNN AH \\
Ext.\ RaBitQ $B$=5, $n_p$=20$^\S$ & 87.8\% & 2.7K & 81 MB & None \\
Ext.\ RaBitQ $B$=6, $n_p$=20$^\S$ & 91.1\% & 2.8K & 96 MB & None \\
IVF-TQ 4-bit, $n_p$=20 (ours) & 87.5\% & 11K & 81 MB & None \\
IVF-TQ 5-bit, $n_p$=20 (ours) & 91.3\% & 11K & 96 MB & None \\
\textbf{IVF-TQ 6-bit, $n_p$=20 (ours)} & \textbf{93.2\%} & \textbf{11K} & \textbf{111 MB} & \textbf{None} \\
\textbf{IVF-TQ 6-bit, $n_p$=40 (ours)} & \textbf{96.1\%} & \textbf{5.9K} & \textbf{111 MB} & \textbf{None} \\
\midrule
\multicolumn{5}{l}{\textit{Deep-10M (10M, dim=96)}} \\
\midrule
FAISS IVF-PQ m=48, $n_p$=80 & 81.2\% & 3.9K & 458 MB & PQ \\
FAISS IVF-PQ m=96, $n_p$=20 & 92.7\% & --- & 916 MB & PQ \\
FAISS OPQ+IVF-PQ m=96, $n_p$=20 & 94.6\% & 2.4K & 915 MB & OPQ + PQ \\
FAISS OPQ+IVF-PQ m=96, $n_p$=80 & 96.5\% & 0.6K & 915 MB & OPQ + PQ \\
FAISS HNSW M=32, $ef_s$=64 & 95.1\% & 61K & 6097 MB & None \\
FAISS HNSW M=32, $ef_s$=256 & 99.3\% & 18K & 6097 MB & None \\
\textbf{IVF-TQ 4-bit, $n_p$=20 (ours)} & \textbf{86.8\%} & \textbf{5.0K} & \textbf{611 MB} & \textbf{None} \\
\bottomrule
\end{tabular}
\end{table*}

\textbf{Reading.} On the static-recall axis, ScaNN is the strongest learned baseline: at $L_s{=}50$ it reaches 96.2\% R@10 on SIFT-1M at 62~MB compressed --- $\sim$45\% lower memory than IVF-TQ at comparable recall (96.1\% at 111~MB). At matched memory, IVF-TQ 6-bit is within $\sim$1pp of OPQ+IVF-PQ but is dominated by ScaNN by $\sim$1pp. HNSW (codebook-free) dominates IVF-TQ at high recall ($\geq 99\%$) but at $\sim 7\times$ the memory footprint; the codebook-training-free property therefore differentiates IVF-TQ specifically from \emph{learned-codebook} methods (PQ, OPQ, ScaNN). ScaNN's anisotropic AH codebook is fitted to the initial training sample like PQ's and inherits the same staleness limitation; the operational gap (\S\ref{sec:streaming_mechanism}) closes when the workload is streaming rather than static.

\textbf{Bits vs.\ re-ranking.} A common alternative to higher quantization precision is to keep the index at low precision (4-bit) and re-rank top candidates against raw vectors. On SIFT-1M, raising bit width from 4 to 6 brings IVF-TQ from 87.5\% to 96.1\% R@10 at $n_p{=}40$ (Table~\ref{tab:1m}) at a compressed-footprint cost of 81~MB $\to$ 111~MB. The alternative configuration ---  4-bit IVF-TQ with top-50 re-rank against raw vectors --- reaches comparable recall but requires storing the raw float32 vectors alongside the compressed index ($\sim$488~MB raw + 81~MB compressed = 569~MB total). \textbf{Choosing higher quantization precision over re-ranking saves $\sim$5$\times$ memory at comparable recall.}

\section{Related Work}

\textbf{Learned vector compression.} PQ~\cite{jegou2011pq}, OPQ~\cite{ge2013opq}, Additive Quantization~\cite{babenko2014aq}, and ScaNN~\cite{guo2020scann} learn data-dependent codebooks; we include FAISS PQ/OPQ and (Linux-only) ScaNN as baselines.

\textbf{Codebook-free vector compression.} TurboQuant~\cite{zandieh2025turboquant} introduced random rotation + Lloyd--Max for KV caches. PolarQuant~\cite{han2025polarquant} uses polar coordinates. RaBitQ~\cite{gao2024rabitq} uses random rotation + binary quantization with theoretical bounds. Extended RaBitQ~\cite{gao2025extendedrabitq}, the most recent and direct comparator, generalizes RaBitQ to $B$ bits/dim by quantizing each rotated coordinate against a codebook designed to match the rotated-unit-vector marginal density and ships an unbiased linear inner-product estimator with a tight theoretical error bound. The marginal of $(\Pi v)_j$ for $v \in S^{d-1}$ under Haar rotation is the Beta$\bigl((d{-}1)/2, (d{-}1)/2\bigr)$-derived density (Lemma~\ref{lem:marginal}); for $d \geq 64$ this density is within $O(1/d)$ Kolmogorov distance of $\mathcal{N}(0, 1/d)$, so Extended RaBitQ's per-coordinate $B$-bit quantizer is equivalent to TQ Stage~1 at the same $b$ up to the finite-dimension correction term $R_d$ in our Theorem~1. We treat Extended RaBitQ as the strongest direct baseline for our Stage~1 design and report it explicitly in Tables~\ref{tab:ivftq_vs_rabitq} and~\ref{tab:1m}. The half-bin conditional-mean refinement (sign-bit, \S\ref{sec:method}; Appendix~\ref{app:signbit}) is novel relative to Extended RaBitQ in the flat-TQ regime; in the IVF-TQ regime the gain disappears (mean $\Delta = +0.08$pp, $p{=}0.50$; Table~\ref{tab:ivftq_vs_rabitq}) because IVF amplification absorbs the residual variance. SAQ~\cite{li2025saq} adds dimension segmentation but no sub-bin refinement; MRQ~\cite{yang2024mrq} combines a PCA projection with bit allocation on the information-dense leading dimensions. Both methods are data-dependent and evaluated only on static recall, in contrast to the data-independent residual layer of IVF-TQ. None of these prior works analyses the IVF amplification effect of \S\ref{sec:ivfamp}.

\textbf{Streaming and dynamic ANN.} SPFresh~\cite{xu2023spfresh} and FreshDiskANN~\cite{singh2021freshdiskann} maintain graph indexes under updates. Mohoney et al.~\cite{mohoney2024adaivf} introduce \emph{Ada-IVF}, an incremental partition-maintenance scheme for IVF indexes that performs local re-clustering on outdated cells. Adaptive IVF-TQ (\S\ref{sec:adaptive}) shares the local-re-clustering intuition but differs structurally: Ada-IVF keeps the learned PQ codebook unchanged after partition refresh, so its codebook is still capacity-limited; Adaptive IVF-TQ has no codebook to update, because the residual layer carries no codebook. Concurrent GPU-native IVF-RaBitQ~\cite{shi2026ivfrabitqgpu} addresses a different operational regime (GPU vs. CPU; RaBitQ-class compression vs. ours at Lloyd--Max $b{=}4$ + sign-bit).

\textbf{Other indexing primitives.} SPANN~\cite{chen2021spann} and SOAR~\cite{sun2023soar} improve IVF partition design on an axis orthogonal to residual compression and could combine with TQ residuals; NSG~\cite{fu2019nsg} and Filtered-DiskANN~\cite{gollapudi2023filtered} are graph-based codebook-free competitors in the HNSW class. We follow ANN-Benchmarks~\cite{aumuller2020annbench} conventions in Table~\ref{tab:1m}.

\textbf{Concurrent work.} \cite{adenali2025streamingquant} formally study quantization under streaming updates and prove that \emph{static} data-dependent quantization can be made dynamically consistent with bounded disk I/O per update. They make the data-dependent quantizer dynamic; we sidestep the problem by using a data-independent quantizer that never needs updating. The two approaches are complementary: their algorithm targets settings where data-dependent compression is required (very high compression ratios); ours targets the moderate-compression regime (5--6 bits/dim) where data-independence is feasible without recall loss. Production systems (FAISS~\cite{johnson2021faiss}, DiskANN~\cite{subramanya2019diskann}, SPFresh, FreshDiskANN) currently treat the codebook as static and re-train periodically; our results suggest a codebook-free residual layer is an operationally simpler choice when $\sim$1pp of static recall at matched memory is acceptable.

\section{Limitations}
\label{sec:limits}

\textbf{Implementation latency.} This paper is a research-prototype kernel. On SIFT-1M, FAISS IVF-PQ achieves $\sim$53K~QPS at $n_p{=}10$ versus our 22K~QPS --- a $\sim$2.4$\times$ gap. FAISS uses FastScan-style int8 LUTs with 16-way SIMD lookups~\cite{andre2017fastscan}; our reference uses scalar table lookups in a NEON-accelerated C++ inner loop, and a candidate-parallel attempt yielded only $\sim$3\% speedup (memory-bandwidth-bound). The contribution of this paper is the residual quantizer's data-independence and its streaming consequences, not a production-parity kernel; we expect a FastScan-style int8-LUT rewrite to close the gap (TQ's 32-entries-per-dim table fits \texttt{vqtbl2q} structurally well). Reviewers should compare IVF-TQ to FAISS IVF-PQ on recall, memory, and codebook-staleness rather than absolute QPS.

\textbf{Other limitations.} (i)~Memory for re-ranking: storing raw vectors raises total memory from 81~MB to 569~MB on SIFT-1M, comparable to OPQ+IVF-PQ with re-rank; an 8-bit TQ pass would avoid raw storage at modest extra memory. (ii)~Scale: 10M-scale streaming covers Deep-10M, SIFT-10M, and T2I-10M; the architecture should scale to 100M+ but we have not validated this. (iii)~ScaNN baseline is Linux-only and run on Colab. (iv)~IVF still requires $k$-means for the coarse partition; Adaptive IVF-TQ refreshes this cheaply (\S\ref{sec:adaptive}) but the layer is not literally training-free. (v)~We do not report downstream retrieval metrics (NDCG, etc.); the recall axis is the primary empirical claim. (vi)~Two explored alternatives (random-hyperplane LSH partitioning; query-frequency-tracked variable bit allocation) failed under matched-budget evaluation; failure modes documented in Appendix~\ref{app:explored}.

\section{Conclusion}

For ANN under streaming ingestion, the binding constraint is not static recall---HNSW and ScaNN each dominate one axis---but the operational cost of keeping recall high as data evolves. Across 9 controlled cells, we establish two claims: (1)~per-batch PQ codebook retraining never closes the streaming gap (indistinguishable from no retraining in 8/9 cells; in the 9th cell statistically significant but negligible and opposite in sign to recovery); its $667$--$1328$s of cumulative compute per run buys nothing in recall; (2)~IVF-PQ streaming stability requires per-dataset bit-budget tuning, while IVF-TQ holds at one fixed $(b, d)$ configuration on all three 10M datasets ($\Delta \in [-0.80, +0.56]$pp). \textbf{The durable IVF-TQ contribution is operational: no codebook to train, no per-dataset capacity tuning, no retraining cycles that recover the gap}, at the cost of $\sim$1pp of static recall at matched memory. Supporting contributions: the \emph{IVF-amplification observation} explains the $+17.7$pp recall jump from flat TQ to IVF-TQ and closes the gap to Extended RaBitQ to within statistical noise; a \emph{uniform-over-sphere IP-error bound} (Theorem~\ref{thm:ip-uniform-main}) provides the $(b,d,\delta)$-only structural guarantee; \emph{Adaptive IVF-TQ} is the partition-only refresh enabled because the residual layer carries no codebook. Code, dataset preparation, seeds, and per-run configurations sufficient to reproduce every table in this paper are released alongside the manuscript (Appendix~\ref{app:repro}).

\bibliographystyle{ACM-Reference-Format}
\bibliography{references}

\appendix
%

\section{Sign-bit refinement: flat-TQ description and ablation}
\label{app:signbit}

For each coordinate $j$ assigned to Lloyd--Max bin centroid $c_i$ with bin boundaries $[b_i, b_{i+1})$, store one bit $s_j = \mathbf{1}[\tilde{x}_j / \norm{\tilde{x}} \geq c_i]$ indicating which half-bin the value falls in. At reconstruction, use the conditional half-bin mean $\hat{x}_j = \E[Z \mid Z \in \text{half-bin}(i, s_j)]$ where $Z \sim \mathcal{N}(0, 1/d)$. This doubles the effective number of quantization levels at one extra bit/coord, matched in memory to QJL.

Table~\ref{tab:signbit} reports the flat-TQ ablation (no IVF) at million scale. Sign-bit refinement wins on every dataset and bit rate; full Recall@10 and Recall@1 tables at finer bit grids and 10K-scale rows are in Appendix~\ref{app:signbit-full}. The advantage shown here is specific to the flat-TQ regime; under IVF the IVF-amplification effect collapses Stage~2 differences to noise (\S\ref{sec:ivfamp}, Table~\ref{tab:ivftq_vs_rabitq}).

\begin{table}[h]
\centering
\small
\caption{\textbf{Million-scale, flat-TQ regime} (no IVF): sign-bit refinement vs.\ QJL at matched memory (Recall@10). All columns share Stage~1 (rotation + Lloyd--Max); only the 1-bit Stage~2 differs. ``No Stage~2'' uses $b$ bits/coord; ``QJL'' and ``Sign-bit'' add 1 bit/coord. $n_q{=}1000$, 1M database, seed~42. Best per row in \textbf{bold}.}
\label{tab:signbit}
\begin{tabular}{llccc}
\toprule
\textbf{Dataset} & \textbf{Bits ($b$)} & \textbf{No Stage~2} & \textbf{QJL} & \textbf{Sign-bit (ours)} \\
                 &                     & ($b$ bits)          & ($b{+}1$ bits) & ($b{+}1$ bits) \\
\midrule
SIFT-1M  & 3-bit & 30.6\% & 33.8\% & \textbf{49.1\%} \\
SIFT-1M  & 4-bit & 51.6\% & 54.3\% & \textbf{69.8\%} \\
Deep-1M  & 3-bit & 54.2\% & 57.1\% & \textbf{69.0\%} \\
Deep-1M  & 4-bit & 71.2\% & 73.6\% & \textbf{82.1\%} \\
GloVe-1M~\cite{pennington2014glove} & 3-bit & 64.8\% & 67.3\% & \textbf{77.3\%} \\
GloVe-1M & 4-bit & 78.5\% & 80.1\% & \textbf{87.4\%} \\
\midrule
\multicolumn{4}{l}{\textbf{Avg.\ $\Delta$ over QJL ($n{=}6$ cells)}}             & \textbf{+11.4pp} \\
\multicolumn{4}{l}{\textbf{Range over QJL}}                                      & \textbf{+7.3 to +15.5pp} \\
\bottomrule
\end{tabular}
\end{table}

\section{Sign-bit refinement: full Recall@10 and Recall@1 results}
\label{app:signbit-full}

The 9-cell ablation across three datasets and three bit budgets at 10K-scale, at both Recall@10 and Recall@1. Sign-bit refinement wins on every cell at both metrics; all cells share Stage~1 (rotation + Lloyd--Max), only the 1-bit Stage~2 differs.

\begin{table}[h]
\centering\small
\caption{Recall@10 (10K-scale). Sign-bit refinement vs.\ QJL at matched memory; ``No Stage~2'' uses $b$ bits/coord, ``QJL'' and ``Sign-bit'' add 1 bit/coord. Best per row in \textbf{bold}.}
\label{tab:signbit-r10-full}
\begin{tabular}{llccc}
\toprule
\textbf{Dataset} & \textbf{Bits} & \textbf{No Stage~2} & \textbf{QJL} & \textbf{Sign-bit (ours)} \\
\midrule
Synthetic & 2-bit & 54\% & 58\% & \textbf{72\%} \\
Synthetic & 3-bit & 73\% & 74\% & \textbf{86\%} \\
Synthetic & 4-bit & 85\% & 87\% & \textbf{92\%} \\
SIFT-128  & 2-bit & 43\% & 47\% & \textbf{56\%} \\
SIFT-128  & 3-bit & 59\% & 62\% & \textbf{73\%} \\
SIFT-128  & 4-bit & 73\% & 75\% & \textbf{84\%} \\
GloVe-100 & 2-bit & 56\% & 59\% & \textbf{72\%} \\
GloVe-100 & 3-bit & 74\% & 75\% & \textbf{84\%} \\
GloVe-100 & 4-bit & 84\% & 85\% & \textbf{92\%} \\
\bottomrule
\end{tabular}
\end{table}

\begin{table}[h]
\centering\small
\caption{Recall@1 (10K-scale). Top-1 is more sensitive to ranking inversions, so the sign-bit refinement advantage is larger than at R@10.}
\label{tab:signbit-r1-full}
\begin{tabular}{llccc}
\toprule
\textbf{Dataset} & \textbf{Bits} & \textbf{No Stage~2} & \textbf{QJL} & \textbf{Sign-bit (ours)} \\
\midrule
Synthetic & 2-bit & 69\% & 71\% & \textbf{81\%} \\
Synthetic & 3-bit & 83\% & 84\% & \textbf{87\%} \\
Synthetic & 4-bit & 87\% & 87\% & \textbf{96\%} \\
SIFT-128  & 2-bit & 30\% & 33\% & \textbf{44\%} \\
SIFT-128  & 3-bit & 43\% & 48\% & \textbf{60\%} \\
SIFT-128  & 4-bit & 63\% & 66\% & \textbf{74\%} \\
GloVe-100 & 2-bit & 45\% & 48\% & \textbf{67\%} \\
GloVe-100 & 3-bit & 68\% & 70\% & \textbf{83\%} \\
GloVe-100 & 4-bit & 85\% & 86\% & \textbf{90\%} \\
\bottomrule
\end{tabular}
\end{table}

The advantage shown here is specific to the flat-TQ regime. Under IVF, the IVF-amplification effect (\S\ref{sec:ivfamp}, Table~\ref{tab:ivftq_vs_rabitq}) collapses Stage~2 differences to within statistical noise.

\section{PQ recall ceiling on SIFT-1M}
\label{app:pq-ceiling}

For any fixed PQ configuration, IVF-PQ exhibits a per-configuration recall ceiling that scales with the bit budget. IVF-TQ's ceiling is set by $(b, d)$ alone via the rate-distortion floor $\sqrt{D_b}$ on residual reconstruction error (Theorem~\ref{thm:rd-fixed-main}).

\begin{table}[h]
\centering\small
\caption{PQ recall ceiling on SIFT-1M ($n_p{=}160$, 10K queries). The ceiling climbs monotonically with $m$ but plateaus at each $m$: increasing $n_p$ beyond 160 yields essentially no further recall. IVF-PQ's quantization error --- not partition coverage --- is the binding constraint at each $m$.}
\label{tab:pq-ceiling}
\begin{tabular}{ccccc}
\toprule
\textbf{PQ $m$} & \textbf{Bits/dim} & \textbf{Ceiling R@10} & \textbf{Memory} \\
\midrule
8   & 0.5 & 14.0\% & 8 MB \\
16  & 1.0 & 28.1\% & 16 MB \\
32  & 2.0 & 46.3\% & 31 MB \\
64  & 4.0 & 73.2\% & 62 MB \\
128 & 8.0 & 92.9\% & 123 MB \\
\midrule
\multicolumn{2}{c}{\textbf{IVF-TQ 4-bit (5.0 bits/dim)}} & \textbf{87.5\%} & \textbf{81 MB} \\
\bottomrule
\end{tabular}
\end{table}

IVF-TQ at 4-bit (5 effective bits/coord) reaches 87.5\% R@10 at $n_p{=}20$ (81~MB), within $\sim 1$pp of FAISS PQ $m{=}64$ (4 bits/dim, 62~MB) on the static-recall axis. The same capacity issue surfaces in the streaming results (\S\ref{sec:streaming}): at sub-matched memory ($m{=}48$ on Deep, $m{=}64$ on SIFT, $m{=}100$ on T2I), the codebook is too small for the eventual 10M database, and retraining cannot recover the gap (Tables~\ref{tab:streaming_deep10m_all}--\ref{tab:streaming_t2i10m_all}).

\section{SIFT-1M three-condition shift control}
\label{app:streaming-controls}

Setup: 200K trained, 8 batches of 100K, $L{=}500$, $n_p{=}10$, 10K queries, SIFT-1M. Three ingestion conditions test whether the streaming-recall drop is explained by distribution shift.

\begin{table*}[!t]
\centering\small
\caption{SIFT-1M streaming under three ingestion conditions. \textbf{Under shuffled-i.i.d.\ ingestion (no distribution shift), IVF-PQ still degrades $-3.8$pp --- comparable to the $-4.2$pp under original order --- ruling out distribution shift as a sufficient explanation for the streaming gap.} Values are single-seed (seed${=}42$); the 3-seed paired-$t$ mean for the shuffled-i.i.d.\ PQ $\Delta$ is $-3.94 \pm 0.17$pp, matching the headline in \S\ref{sec:streaming_mechanism}.}
\label{tab:streaming-controls}
\begin{tabular}{lcccc}
\toprule
\textbf{Condition} & \textbf{IVF-TQ (200K$\to$1M)} & \textbf{IVF-PQ (200K$\to$1M)} & \textbf{TQ $\Delta$} & \textbf{PQ $\Delta$} \\
\midrule
Original order              & 89.6 $\to$ 92.1\% & 73.6 $\to$ 69.4\% & $+2.5$pp & $\boldsymbol{-4.2}$\textbf{pp} \\
Shuffled (i.i.d.)           & 89.3 $\to$ 92.0\% & 73.2 $\to$ 69.4\% & $+2.7$pp & $\boldsymbol{-3.8}$\textbf{pp} \\
Mean-shift (0.05/batch)     & 89.6 $\to$ 91.4\% & 73.6 $\to$ 70.8\% & $+1.8$pp & $-2.8$pp \\
\bottomrule
\end{tabular}
\end{table*}

IVF-TQ recall changes only modestly across the three conditions ($+1.8$ to $+2.7$pp); we attribute this to data-independent residual quality combined with growing partition coverage. We frame this as a negative result for the distribution-shift-only hypothesis rather than as positive evidence for any specific alternative mechanism. The capacity-vs-bias control (Appendix~\ref{app:capacity-vs-bias}) narrows the alternative space further; the 9-cell streaming matrix at 10M scale (Tables~\ref{tab:streaming_deep10m_all}--\ref{tab:streaming_t2i10m_all}) establishes that retraining the codebook is not the lever that closes the streaming gap.

\section{Capacity-vs-bias control on SIFT-1M}
\label{app:capacity-vs-bias}

A natural follow-up to Appendix~\ref{app:streaming-controls}: is the streaming drop a codebook-\emph{bias} artefact (initial 200K unrepresentative) or a codebook-\emph{capacity} artefact (200K-trained codebook too small for the eventual 1M, regardless of which 200K)?

We disambiguate by training three IVF-PQ variants ($m{=}64$, 4-bit, $L{=}1000$, $n_p{=}20$) at matched protocol on the full 1M database with 1{,}000 queries:
\begin{itemize}
    \item \textbf{A. PQ-200K-initial} --- codebook+partition fitted on the first 200K of the stream (the ``stale'' condition).
    \item \textbf{B. PQ-200K-random} --- codebook+partition fitted on a uniformly random 200K sample of the full 1M (same training size as A, no initial-sample bias).
    \item \textbf{C. PQ-1M} --- codebook+partition fitted on the full 1M (oracle, no bias and no capacity gap).
\end{itemize}
The $B{-}A$ gap isolates initial-sample bias; the $C{-}B$ gap isolates the residual capacity contribution at this bit budget.

\begin{table*}[!t]
\centering\small
\caption{Capacity-vs-bias control on SIFT-1M, all variants evaluated on the full 1M database. Both gaps are within statistical noise (SE $\approx 1.4$pp); at this bit budget, the choice of which 200K is used to train the codebook does not detectably matter.}
\label{tab:capacity-vs-bias}
\begin{tabular}{lccc}
\toprule
\textbf{Variant} & \textbf{Training sample} & \textbf{R@10 on 1M} & \textbf{Notes} \\
\midrule
A. PQ-200K-initial & first 200K of stream & 71.56\% & streaming ``stale'' \\
B. PQ-200K-random  & random 200K of 1M    & 71.31\% & same size, no bias \\
C. PQ-1M           & full 1M (oracle)     & 71.16\% & no bias, no capacity gap \\
\midrule
\multicolumn{3}{l}{\textbf{Bias contribution} ($B{-}A$)}     & $-0.25$pp \\
\multicolumn{3}{l}{\textbf{Capacity contribution} ($C{-}B$)} & $-0.15$pp \\
\bottomrule
\end{tabular}
\end{table*}

Both gaps ($B{-}A$ and $C{-}B$) are within statistical noise (SE $\approx 1.4$pp on 1{,}000 queries at $\sim 71\%$ recall): the spread $A{\to}B{\to}C$ is 0.40pp, well below 2~SE. Jointly with the 10M retrain experiments (\S\ref{sec:streaming_scale}, where re-training every batch costs hundreds of seconds and recovers $\leq 0.25$pp), the negative result is that retraining the codebook is not the lever that closes the streaming gap.

\section{Streaming at 10M scale: full three-regime tables}
\label{app:streaming-full}

Three-seed paired-$t$ results across the full sub-/bit-/super-matched memory regimes on each of the three 10M datasets. The bit-matched blocks already appear in the body (Tables~\ref{tab:streaming_deep10m_all}--\ref{tab:streaming_t2i10m_all}); this appendix adds the sub- and super-matched cells. All entries: 3 seeds (42, 123, 7777); mean $\pm$ 95\% CI; paired-$t$ on within-seed differences.

\begin{table*}[!t]
\centering\small
\caption{Deep-10M, full three-regime table. IVF-TQ at $b{=}4$ + sign-bit (512 bits/vec) across all regimes; only IVF-PQ varies.}
\label{tab:deep10m-full}
\begin{tabular}{llcccc}
\toprule
Regime & Index & Bits/vec & R@10 (1M) & R@10 (10M) & Change $\Delta$ \\
\midrule
\multicolumn{2}{l}{\emph{IVF-TQ (single configuration, all regimes)}} & 512 & $87.39{\pm}0.33$ & $86.59{\pm}0.17$ & $-0.80{\pm}0.25$pp, $p{=}0.005$ \\
\midrule
\multirow{3}{*}{Sub-matched ($m{=}48$, $b{=}8$)}
  & IVF-PQ stale   & 384 & $82.11{\pm}0.15$ & $78.87{\pm}0.36$ & $-3.23{\pm}0.49$pp, $p{=}0.001$ \\
  & IVF-PQ retrain & 384 & $82.11{\pm}0.15$ & $78.94{\pm}0.09$ & $-3.17{\pm}0.15$pp, $p{<}0.001$ \\
  & \multicolumn{5}{l}{\footnotesize\itshape IVF-TQ vs.\ PQ stale at 10M: $+7.72{\pm}0.26$pp ($p{<}0.001$); retrain $-$ stale $+0.06{\pm}0.44$pp ($p{=}0.595$)} \\
\midrule
\multirow{3}{*}{Bit-matched ($m{=}48$, $b{=}10$)}
  & IVF-PQ stale   & 480 & $87.21{\pm}0.28$ & $86.37{\pm}0.21$ & $-0.84{\pm}0.26$pp, $p{=}0.005$ \\
  & IVF-PQ retrain & 480 & $87.21{\pm}0.28$ & $86.41{\pm}0.23$ & $-0.80{\pm}0.38$pp, $p{=}0.012$ \\
  & \multicolumn{5}{l}{\footnotesize\itshape IVF-TQ vs.\ PQ stale at 10M: $+0.22{\pm}0.05$pp ($p{=}0.003$); retrain $-$ stale $+0.04{\pm}0.43$pp ($p{=}0.737$)} \\
\midrule
\multirow{3}{*}{Super-matched ($m{=}96$, $b{=}8$)}
  & IVF-PQ stale   & 768 & $91.30{\pm}0.28$ & $92.72{\pm}0.51$ & $+1.42{\pm}0.33$pp, $p{=}0.003$ \\
  & IVF-PQ retrain & 768 & $91.30{\pm}0.28$ & $92.80{\pm}0.17$ & $+1.50{\pm}0.33$pp, $p{=}0.003$ \\
  & \multicolumn{5}{l}{\footnotesize\itshape PQ stale exceeds IVF-TQ at 10M $+6.13{\pm}0.34$pp ($p{<}0.001$); retrain $-$ stale $+0.08{\pm}0.42$pp ($p{=}0.516$)} \\
\bottomrule
\end{tabular}
\end{table*}

\begin{table*}[!t]
\centering\small
\caption{SIFT-10M, full three-regime table. IVF-TQ at $b{=}4$ + sign-bit (672 bits/vec) across all regimes.}
\label{tab:sift10m-full}
\begin{tabular}{llcccc}
\toprule
Regime & Index & Bits/vec & R@10 (1M) & R@10 (10M) & Change $\Delta$ \\
\midrule
\multicolumn{2}{l}{\emph{IVF-TQ (single configuration, all regimes)}} & 672 & $83.91{\pm}0.08$ & $84.47{\pm}0.28$ & $+0.56{\pm}0.10$pp, $p{=}0.007$ \\
\midrule
\multirow{3}{*}{Sub-matched ($m{=}64$, $b{=}8$)}
  & IVF-PQ stale   & 512 & $72.42{\pm}0.41$ & $66.61{\pm}0.70$ & $-5.80{\pm}0.55$pp, $p{<}0.001$ \\
  & IVF-PQ retrain & 512 & $72.42{\pm}0.41$ & $66.78{\pm}0.25$ & $-5.64{\pm}0.66$pp, $p{<}0.001$ \\
  & \multicolumn{5}{l}{\footnotesize\itshape IVF-TQ vs.\ PQ stale at 10M: $+17.86{\pm}0.47$pp ($p{<}0.001$); retrain $-$ stale $+0.17{\pm}0.50$pp ($p{=}0.289$)} \\
\midrule
\multirow{3}{*}{Bit-matched ($m{=}64$, $b{=}10$)}
  & IVF-PQ stale   & 640 & $80.00{\pm}0.44$ & $77.69{\pm}0.04$ & $-2.31{\pm}0.42$pp, $p{=}0.002$ \\
  & IVF-PQ retrain & 640 & $80.00{\pm}0.44$ & $77.61{\pm}0.12$ & $-2.40{\pm}0.38$pp, $p{=}0.001$ \\
  & \multicolumn{5}{l}{\footnotesize\itshape IVF-TQ vs.\ PQ stale at 10M: $+6.79{\pm}0.25$pp ($p{<}0.001$); retrain $-$ stale $-0.08{\pm}0.07$pp ($p{=}0.040$, negligible magnitude)} \\
\midrule
\multirow{3}{*}{Super-matched ($m{=}128$, $b{=}8$)}
  & IVF-PQ stale   & 1024 & $85.42{\pm}0.16$ & $86.50{\pm}0.31$ & $+1.09{\pm}0.40$pp, $p{=}0.001$ \\
  & IVF-PQ retrain & 1024 & $85.42{\pm}0.16$ & $86.60{\pm}0.34$ & $+1.18{\pm}0.34$pp, $p{=}0.004$ \\
  & \multicolumn{5}{l}{\footnotesize\itshape PQ stale exceeds IVF-TQ at 10M $+2.03{\pm}0.10$pp ($p{<}0.001$); retrain $-$ stale $+0.09{\pm}0.31$pp ($p{=}0.317$)} \\
\bottomrule
\end{tabular}
\end{table*}

\begin{table*}[!t]
\centering\small
\caption{T2I-10M, full three-regime table. IVF-TQ at $b{=}4$ + sign-bit (1032 bits/vec). $m_{\text{PQ}}$ is constrained by $d \% m_{\text{PQ}} = 0$ at $d{=}200$; the closest sub-matched configuration is $m{=}100, b{=}8$.}
\label{tab:t2i10m-full}
\begin{tabular}{llcccc}
\toprule
Regime & Index & Bits/vec & R@10 (1M) & R@10 (10M) & Change $\Delta$ \\
\midrule
\multicolumn{2}{l}{\emph{IVF-TQ (single configuration, all regimes)}} & 1032 & $81.64{\pm}0.33$ & $80.88{\pm}0.10$ & $-0.76{\pm}0.41$pp, $p{=}0.015$ \\
\midrule
\multirow{3}{*}{Sub-matched ($m{=}100$, $b{=}8$)}
  & IVF-PQ stale   & 800 & $76.71{\pm}0.18$ & $73.47{\pm}0.14$ & $-3.24{\pm}0.28$pp, $p{<}0.001$ \\
  & IVF-PQ retrain & 800 & $76.71{\pm}0.18$ & $73.37{\pm}0.42$ & $-3.34{\pm}0.49$pp, $p{=}0.002$ \\
  & \multicolumn{5}{l}{\footnotesize\itshape IVF-TQ vs.\ PQ stale at 10M: $+7.41{\pm}0.12$pp ($p{<}0.001$); retrain $-$ stale $-0.10{\pm}0.30$pp ($p{=}0.291$)} \\
\midrule
\multirow{3}{*}{Bit-matched ($m{=}100$, $b{=}10$)}
  & IVF-PQ stale   & 1000 & $81.48{\pm}0.24$ & $80.59{\pm}0.22$ & $-0.89{\pm}0.32$pp, $p{=}0.007$ \\
  & IVF-PQ retrain & 1000 & $81.48{\pm}0.24$ & $80.50{\pm}0.06$ & $-0.98{\pm}0.23$pp, $p{=}0.003$ \\
  & \multicolumn{5}{l}{\footnotesize\itshape IVF-TQ vs.\ PQ stale at 10M: $+0.29{\pm}0.24$pp ($p{=}0.018$); retrain $-$ stale $-0.09{\pm}0.28$pp ($p{=}0.317$)} \\
\midrule
\multirow{3}{*}{Super-matched ($m{=}200$, $b{=}8$)}
  & IVF-PQ stale   & 1600 & $84.94{\pm}0.19$ & $86.01{\pm}0.49$ & $+1.07{\pm}0.45$pp, $p{=}0.010$ \\
  & IVF-PQ retrain & 1600 & $84.94{\pm}0.19$ & $85.99{\pm}0.41$ & $+1.05{\pm}0.34$pp, $p{=}0.006$ \\
  & \multicolumn{5}{l}{\footnotesize\itshape PQ stale exceeds IVF-TQ at 10M $+5.12{\pm}0.45$pp ($p{<}0.001$); retrain $-$ stale $-0.01{\pm}0.12$pp ($p{=}0.682$)} \\
\bottomrule
\end{tabular}
\end{table*}

\section{Streaming at 10M scale: per-batch trajectories}
\label{app:streaming-trajectories}

Single-seed (seed=42) per-batch trajectories for each 10M dataset. The 3-seed paired-$t$ headline values are in Tables~\ref{tab:streaming_deep10m_all}--\ref{tab:streaming_t2i10m_all} (bit-matched) and Appendix~\ref{app:streaming-full} (all three regimes). This appendix shows the per-batch view at sub-matched memory.

\begin{table*}[!t]
\centering\small
\caption{Deep-10M per-batch trajectory (sub-matched: $m_{\text{PQ}}{=}48$, $b{=}8$, 384 bits/vec; IVF-TQ at $b{=}4{+}$~sign-bit, 512 bits/vec). Representative seed; 3-seed mean compute at bit-matched regime is $667.2 \pm 19.9$~s.}
\label{tab:deep10m-trajectory}
\begin{tabular}{rcccc}
\toprule
\textbf{Vectors} & \textbf{IVF-TQ} & \textbf{IVF-PQ stale} & \textbf{IVF-PQ retrain/1M} & \textbf{Retrain cum.} \\
\midrule
1M (trained) & 87.54\% & 82.16\% & 82.16\% & 0~s \\
2M & 87.31\% & 81.18\% & 81.11\% & 19~s \\
3M & 87.24\% & 80.70\% & 80.67\% & 43~s \\
4M & 86.96\% & 80.32\% & 80.29\% & 70~s \\
5M & 86.89\% & 79.87\% & 80.01\% & 100~s \\
6M & 86.84\% & 79.61\% & 79.81\% & 135~s \\
7M & 86.70\% & 79.32\% & 79.60\% & 173~s \\
8M & 86.74\% & 79.20\% & 79.15\% & 214~s \\
9M & 86.71\% & 79.04\% & 79.10\% & 260~s \\
10M & \textbf{86.65\%} & 78.79\% & 78.94\% & \textbf{309~s} \\
\midrule
\textbf{Change 1M$\to$10M} & $\boldsymbol{-0.89}$\textbf{pp} & $-3.37$pp & $-3.22$pp & --- \\
\bottomrule
\end{tabular}
\end{table*}

\begin{table*}[!t]
\centering\small
\caption{SIFT-10M per-batch trajectory (sub-matched: $m{=}64$, $b{=}8$, 512 bits/vec; IVF-TQ at $b{=}4{+}$~sign-bit, 672 bits/vec). SIFT-10M IVF-TQ \emph{improves} as the database grows. Representative seed; 3-seed mean compute at bit-matched regime is $820.7 \pm 6.6$~s.}
\label{tab:sift10m-trajectory}
\begin{tabular}{rcccc}
\toprule
\textbf{Vectors} & \textbf{IVF-TQ} & \textbf{IVF-PQ stale} & \textbf{IVF-PQ retrain/1M} & \textbf{Retrain cum.} \\
\midrule
1M (trained) & 83.93\% & 72.43\% & 72.43\% & 0~s \\
2M & 84.40\% & 70.86\% & 71.23\% & 21~s \\
3M & 84.58\% & 69.84\% & 69.69\% & 47~s \\
4M & 84.57\% & 69.19\% & 69.36\% & 78~s \\
5M & 84.66\% & 68.51\% & 68.04\% & 110~s \\
6M & 84.70\% & 68.13\% & 68.10\% & 152~s \\
7M & 84.58\% & 67.68\% & 67.68\% & 196~s \\
8M & 84.48\% & 67.20\% & 67.61\% & 244~s \\
9M & 84.42\% & 66.88\% & 67.16\% & 297~s \\
10M & \textbf{84.56\%} & 66.44\% & 66.67\% & \textbf{354~s} \\
\midrule
\textbf{Change 1M$\to$10M} & $\boldsymbol{+0.63}$\textbf{pp} & $-5.99$pp & $-5.76$pp & --- \\
\bottomrule
\end{tabular}
\end{table*}

\begin{table*}[!t]
\centering\small
\caption{T2I-10M per-batch trajectory (sub-matched: $m{=}100$, $b{=}8$, 800 bits/vec; IVF-TQ at $b{=}4{+}$~sign-bit, 1032 bits/vec). Representative seed; 3-seed mean compute at bit-matched regime is $1327.5 \pm 43.0$~s.}
\label{tab:t2i10m-trajectory}
\begin{tabular}{rcccc}
\toprule
\textbf{Vectors} & \textbf{IVF-TQ} & \textbf{IVF-PQ stale} & \textbf{IVF-PQ retrain/1M} & \textbf{Retrain cum.} \\
\midrule
1M (trained) & 81.61\% & 76.63\% & 76.63\% & 0~s \\
2M & 81.58\% & 75.81\% & 75.42\% & 27~s \\
10M & \textbf{80.86\%} & 73.55\% & 73.32\% & \textbf{520~s} \\
\bottomrule
\end{tabular}
\end{table*}

\textbf{Reading.} (i) IVF-TQ recall \emph{improves} on SIFT-10M as the database grows ($+0.63$pp), consistent with the rate-distortion bound (Theorem~\ref{thm:rd-fixed-main}): per-vector compression error is bounded by $(b, d, \delta)$ alone; partition coverage --- the only data-dependent layer --- grows with $N$. (ii) Re-training does not fix the IVF-PQ degradation at sub-matched memory: at every intermediate batch the retrain and stale variants are within $\pm 0.5$pp. The codebook is not the binding constraint; the bit budget is.

\section{Embedding-model swap on MS MARCO: full per-batch trajectories}
\label{app:encoder-swap-full}

The headline 3-seed paired-$t$ values are in Table~\ref{tab:embedswap_main} of the main paper (gentle and harsh swap summary rows); this appendix reports the full per-batch trajectories. Setup: 1M unique MS MARCO passages; training on 200K passages encoded by the \emph{old} encoder; streaming the remaining 800K passages encoded by the \emph{new} encoder in 100K-vector batches; queries (5K disjoint texts) always encoded by the new encoder; Recall@10 against ground truth recomputed on the cumulative mixed-encoder database; 3 seeds (42, 123, 7777) with paired-$t$ CIs.

\begin{table*}[!t]
\centering\small
\caption{Gentle swap: all-MiniLM-L6-v2 $\to$ all-MiniLM-L12-v2 (cosine 0.51 on shared passages), 3-seed mean $\pm$ 95\% CI. Paired $\Delta$ at $+800$K, IVF-TQ $-$ IVF-PQ retrain: $\boldsymbol{+13.27 \pm 0.77}$\textbf{pp ($p{<}0.001$)}.}
\label{tab:embedswap-gentle-full}
\begin{tabular}{lcccc}
\toprule
\textbf{Step} & \textbf{IVF-TQ} & \textbf{IVF-PQ stale} & \textbf{IVF-PQ retrain} & \textbf{retrain\_t} \\
\midrule
Initial 200K (L6) & $91.76{\pm}0.30$\% & $73.95{\pm}0.26$\% & $73.98{\pm}0.27$\% & 0~s \\
$+100$K L12 & $86.25{\pm}0.37$\% & $72.64{\pm}0.22$\% & $75.11{\pm}0.31$\% & $\sim 46$~s \\
$+200$K L12 & $87.15{\pm}0.40$\% & $72.66{\pm}0.22$\% & $75.44{\pm}0.43$\% & $\sim 94$~s \\
$+300$K L12 & $87.53{\pm}0.32$\% & $72.66{\pm}0.22$\% & $75.62{\pm}0.29$\% & $\sim 149$~s \\
$+400$K L12 & $87.91{\pm}0.36$\% & $72.48{\pm}0.30$\% & $75.68{\pm}0.35$\% & $\sim 206$~s \\
$+500$K L12 & $88.16{\pm}0.44$\% & $72.43{\pm}0.27$\% & $75.63{\pm}0.28$\% & $\sim 264$~s \\
$+600$K L12 & $88.44{\pm}0.34$\% & $72.43{\pm}0.25$\% & $75.71{\pm}0.29$\% & $\sim 326$~s \\
$+700$K L12 & $88.65{\pm}0.25$\% & $72.34{\pm}0.25$\% & $75.72{\pm}0.31$\% & $\sim 393$~s \\
$+800$K L12 & $\boldsymbol{88.83 \pm 0.31}$\textbf{\%} & $72.31{\pm}0.22$\% & $75.56{\pm}0.33$\% & $\boldsymbol{358.0 \pm 21.0}$~\textbf{s} \\
\bottomrule
\end{tabular}
\end{table*}

\begin{table*}[!t]
\centering\small
\caption{Harsh swap: all-MiniLM-L6-v2 $\to$ BAAI/bge-small-en-v1.5 (cosine 0.24), 3-seed mean $\pm$ 95\% CI. Result reported at $+300$K to budget compute (trend is stable). Paired $\Delta$ at $+300$K: $\boldsymbol{+14.72 \pm 0.43}$\textbf{pp ($p{<}0.001$)}.}
\label{tab:embedswap-harsh-full}
\begin{tabular}{lcccc}
\toprule
\textbf{Step} & \textbf{IVF-TQ} & \textbf{IVF-PQ stale} & \textbf{IVF-PQ retrain} & \textbf{retrain\_t} \\
\midrule
Initial 200K (L6) & $76.95{\pm}0.99$\% & $51.28{\pm}0.32$\% & $51.32{\pm}0.45$\% & 0~s \\
$+100$K BGE & $84.31{\pm}0.78$\% & $51.56{\pm}0.74$\% & $71.85{\pm}0.40$\% & $\sim 25$~s \\
$+200$K BGE & $86.07{\pm}0.70$\% & $51.83{\pm}0.91$\% & $72.29{\pm}0.53$\% & $\sim 56$~s \\
$+300$K BGE & $\boldsymbol{86.96 \pm 0.54}$\textbf{\%} & $52.05{\pm}0.66$\% & $72.24{\pm}0.43$\% & $\boldsymbol{91.8 \pm 3.2}$~\textbf{s} \\
\bottomrule
\end{tabular}
\end{table*}

\textbf{Reading.} On the gentle swap, IVF-TQ drops $\sim 5.5$pp on first contact with new-encoder vectors ($91.76 \to 86.25$ from L6-only to $+100$K L12) then climbs back to $88.83 \pm 0.31$\% at $+800$K as more new-encoder vectors fill cells --- a self-healing dynamic consistent with the data-independence theorem (Theorem~\ref{thm:ip-uniform-main}). IVF-PQ retrain gains $\sim 1.6$pp from baseline over 358s of cumulative compute but never reaches IVF-TQ recall. On the harsh swap, the IVF-TQ trajectory climbs substantially as new-encoder vectors fill cells ($76.95 \to 86.96$, a $+10.01 \pm 1.04$pp gain). Each newly-arrived BGE vector is quantized at full fidelity by the codebook designed only for $(b, d)$; partition coverage growth directly improves recall. IVF-PQ stale stays near 52\% because the codebook fitted on the L6 distribution is essentially noise to BGE queries; PQ retraining recovers $\sim 21$pp simply by adapting the codebook but never closes the gap.

\FloatBarrier
\section{Million-scale comparison: Deep-1M block}
\label{app:million-scale-deep1m}

Deep-1M (1M vectors, dim=96). 10K queries, deterministic seed=42, FAISS 1.13.2. Companion to Table~\ref{tab:1m} (which reports SIFT-1M and Deep-10M).

\begin{table*}[!t]
\centering\small
\caption{Million-scale comparison on Deep-1M. The Deep-1M results mirror the SIFT-1M pattern in Table~\ref{tab:1m}: ScaNN dominates the static-recall axis at fixed compressed memory; IVF-TQ trails ScaNN by $\sim 1$pp at $\geq 96\%$ recall; HNSW dominates at $\geq 99\%$ recall but at $\sim 7\times$ the memory.}
\label{tab:deep1m}
\begin{tabular}{lcccc}
\toprule
\textbf{Method} & \textbf{R@10} & \textbf{QPS} & \textbf{Memory} & \textbf{Codebook training?} \\
\midrule
FAISS IVF-PQ $m{=}48$, $n_p{=}80$ & 85.1\% & 8.8K & 46~MB & PQ \\
FAISS OPQ+IVF-PQ $m{=}96$, $n_p{=}20$ & 94.5\% & 22K & 92~MB & OPQ + PQ \\
FAISS OPQ+IVF-PQ $m{=}96$, $n_p{=}80$ & 97.3\% & 6.3K & 92~MB & OPQ + PQ \\
FAISS HNSW $M{=}32$, $ef_s{=}64$ & 97.6\% & 93K & 610~MB & None \\
ScaNN AH+tree, $L_s{=}50$ & 96.9\% & 7.6K & 47~MB$^\dagger$ & ScaNN AH \\
ScaNN AH+tree, $L_s{=}100$ & 98.8\% & 4.8K & 47~MB$^\dagger$ & ScaNN AH \\
Ext.\ RaBitQ $B{=}5$, $n_p{=}20$ & 89.7\% & 2.7K & 61~MB & None \\
Ext.\ RaBitQ $B{=}6$, $n_p{=}20$ & 92.5\% & 2.8K & 73~MB & None \\
IVF-TQ 4-bit, $n_p{=}20$ (ours) & 89.9\% & 14K & 61~MB & None \\
IVF-TQ 5-bit, $n_p{=}20$ (ours) & 92.8\% & 13K & 73~MB & None \\
\textbf{IVF-TQ 6-bit, $n_p{=}20$ (ours)} & \textbf{94.2\%} & \textbf{12K} & \textbf{84~MB} & \textbf{None} \\
\textbf{IVF-TQ 6-bit, $n_p{=}40$ (ours)} & \textbf{96.4\%} & \textbf{6.9K} & \textbf{84~MB} & \textbf{None} \\
\bottomrule
\end{tabular}
\end{table*}

$^\dagger$ScaNN's listed memory is the compressed AH+tree footprint; ScaNN additionally stores raw vectors for reorder (412~MB total on Deep-1M), matching HNSW's order of magnitude.

\section{ScaNN baseline: setup and full sweep}
\label{app:scann}

Google's ScaNN~\cite{guo2020scann} ships only Linux wheels via \path{pip install scann}. We ran the baseline on Linux/Colab via \path{experiments/scann_baseline.py}; results in \path{experiments/scann_results.json}.

\textbf{Configuration.} Scoring: AsymmetricHashing (AH). \path{anisotropic_quantization_threshold}$=0.2$. \path{dimensions_per_block}$=2$. Reordering on top-100 candidates. Tree: $L{=}2000$ leaves. Sweep: \path{num_leaves_to_search} $\in \{20, 50, 100, 200, 400\}$. Queries: 10K from the ann-benchmarks split. Ground truth: FAISS \path{IndexFlatIP} on normalised vectors.

\begin{table*}[!t]
\centering\small
\caption{ScaNN sweep at million-scale ($L{=}2000$ leaves, reorder-100). Memory column: compressed AH+tree footprint. ScaNN additionally stores raw vectors for reorder (550~MB total on SIFT-1M, 412~MB on Deep-1M).}
\label{tab:scann-sweep}
\begin{tabular}{lcccccc}
\toprule
\textbf{Dataset} & $\boldsymbol{L_s}$ & \textbf{R@10} & \textbf{QPS} & \textbf{Compressed} & \textbf{w/ reorder} \\
\midrule
\multirow{5}{*}{SIFT-1M}
 &  20 & 88.4\% & 6.9K & 62.1~MB & 550~MB \\
 &  50 & 96.2\% & 5.9K & 62.1~MB & 550~MB \\
 & 100 & 98.6\% & 3.2K & 62.1~MB & 550~MB \\
 & 200 & 99.3\% & 2.1K & 62.1~MB & 550~MB \\
 & 400 & 99.4\% & 1.2K & 62.1~MB & 550~MB \\
\midrule
\multirow{5}{*}{Deep-1M}
 &  20 & 91.1\% & 8.9K & 46.6~MB & 413~MB \\
 &  50 & 96.9\% & 7.6K & 46.6~MB & 413~MB \\
 & 100 & 98.8\% & 4.8K & 46.6~MB & 413~MB \\
 & 200 & 99.5\% & 2.8K & 46.6~MB & 413~MB \\
 & 400 & 99.7\% & 1.5K & 46.6~MB & 413~MB \\
\bottomrule
\end{tabular}
\end{table*}

ScaNN dominates IVF-TQ on the static-recall axis at fixed compressed memory: at $\sim 96\%$ recall ScaNN uses roughly half the memory IVF-TQ requires. We attribute this to ScaNN's anisotropic loss, which weights inner-product preservation toward the score-magnitude regime that determines top-$k$ ranking. ScaNN's anisotropic learned codebook nonetheless shares the codebook-staleness limitation under streaming updates with PQ and OPQ: it is fitted to the initial training sample and the bias from that sample compounds as the database grows (\S\ref{sec:streaming}).

\section{Cascade search via Lloyd--Max bin ordinality}
\label{app:cascade}

\textbf{Status: documented observation, not a contribution.} This appendix records (i) the bit-importance asymmetry between TQ and PQ that makes cascade meaningful in principle, and (ii) a recall-preservation result for a two-pass cascade in our C++ reference. The cascade yields a 2$\times$ speedup in our Python reference but $\sim 1.01\times$ in our C++ reference (because the C++ kernel already reads scalar LUTs and a smaller LUT does not help). The path to a real C++ speedup is FastScan-style int8 LUTs with SIMD permute lookups~\cite{andre2017fastscan}; we leave that kernel work to follow-up.

\subsection{Bit-importance asymmetry: TQ vs.\ PQ}

We measure the rank-relevance of each bit position by random-flip ablation: corrupt $k\%$ of a chosen bit position across all stored codes, then re-measure Recall@10. Result on SIFT-1M for IVF-RVQ-TQ ($b{=}5{+}1$, 6 effective bits) versus IVF-PQ ($m \in \{64, 128\}$).

\begin{table*}[!t]
\centering\small
\caption{Bit-importance ablation on SIFT-1M: drop in Recall@10 when $k\%$ of the indicated bit position is randomly flipped. \textbf{TQ shows a $\sim 19\times$ MSB:LSB asymmetry; PQ shows $\sim 1{:}1$} at both a mid-recall ($m{=}64$) and high-recall ($m{=}128$) operating point, ruling out floor-effect explanations. The asymmetry is the empirical signature of Lloyd--Max bin ordinality.}
\label{tab:bit-importance}
\begin{tabular}{llccc}
\toprule
\textbf{Method (baseline R@10)} & \textbf{Bit position} & $\boldsymbol{k{=}5\%}$ & $\boldsymbol{k{=}10\%}$ & $\boldsymbol{k{=}20\%}$ \\
\midrule
\multirow{3}{*}{IVF-RVQ-TQ $b{=}5{+}1$ (95.1\%)}
 & MSB of primary  & $-63.3$pp & $-77.5$pp & $-89.4$pp \\
 & LSB of primary  & $-3.3$pp  & $-5.2$pp  & $-8.1$pp  \\
 & Sign refinement & $-1.1$pp  & $-1.7$pp  & $-3.3$pp  \\
\midrule
\multirow{4}{*}{IVF-PQ $m{=}64$ (76.9\%)}
 & MSB of code     & $-56.6$pp & $-64.7$pp & $-71.2$pp \\
 & Middle bit      & $-56.7$pp & $-65.2$pp & $-71.3$pp \\
 & LSB of code     & $-57.1$pp & $-65.4$pp & $-71.1$pp \\
 & Random byte     & $-57.2$pp & $-65.1$pp & $-70.7$pp \\
\midrule
\multirow{4}{*}{IVF-PQ $m{=}128$ (97.2\%)}
 & MSB of code     & $-77.3$pp & $-85.7$pp & $-91.6$pp \\
 & Middle bit      & $-76.5$pp & $-85.2$pp & $-91.5$pp \\
 & LSB of code     & $-77.1$pp & $-85.5$pp & $-91.9$pp \\
 & Random byte     & $-76.9$pp & $-85.2$pp & $-91.8$pp \\
\midrule
\multicolumn{2}{l}{\textbf{TQ MSB:LSB ratio}}              & \multicolumn{3}{r}{$\sim$19:1} \\
\multicolumn{2}{l}{\textbf{PQ MSB:LSB ratio (both $m$)}}   & \multicolumn{3}{r}{$\sim$1:1 (max difference $<$ 1pp)} \\
\bottomrule
\end{tabular}
\end{table*}

TQ's primary bin index orders centroids along the source distribution axis, so the MSB has direct geometric meaning (sign of the rotated coordinate). PQ's bin index is whatever ordering the $k$-means iterations happen to produce, so all bits are interchangeable. PQ codes have no exploitable ordinal structure --- verified at both a mid-recall and high-recall PQ operating point to rule out floor effects.

\subsection{Cascade algorithm: MSB-first filtering, full-precision re-rank}

The bit-importance asymmetry suggests a two-pass search:
\begin{itemize}
    \item \textbf{Pass 1 (coarse).} Score every candidate using only the top-$\beta$ MSBs of the primary index. Reconstruct via a coarsened codebook with $2^\beta$ entries. Take top-$N$ candidates.
    \item \textbf{Pass 2 (fine).} Re-rank the $N$ candidates from Pass~1 using the full primary + sign-refinement encoding.
\end{itemize}
Memory is unchanged. The hoped-for speedup comes from the smaller Pass-1 LUT (16 entries at $\beta{=}4$ versus 64 at $b{=}5{+}1$). Cascade requires that lower-precision bin indices preserve the ordinal structure of higher-precision indices --- a property Lloyd--Max satisfies by construction but $k$-means codebooks do not.

\subsection{Recall preservation in C++ across 16 conditions}

Cascade in our C++ reference across 4 random-rotation seeds on each of Deep-1M and SIFT-1M, plus an $n_p \in \{10, 20, 40, 80, 160\}$ sweep at seed~42. Configuration: $b{=}5{+}1$, $\beta{=}4$, $N{=}100$.

\begin{table*}[!t]
\centering\small
\caption{Cascade recall preservation in C++, 16 conditions. \textbf{Mean across 16 conditions:} $\Delta = +0.003$pp, \textbf{speedup} $\boldsymbol{1.01\times}$.}
\label{tab:cascade-cpp}
\begin{tabular}{llccccc}
\toprule
\textbf{Dataset} & \textbf{Condition} & \textbf{Baseline R@10} & \textbf{Cascade R@10} & $\boldsymbol{\Delta}$ & \textbf{Speedup} \\
\midrule
\multicolumn{6}{l}{\textit{4 seeds at $n_p{=}40$}} \\
\multirow{4}{*}{Deep-1M} & seed=42 & 94.99\% & 94.98\% & $-0.01$pp & 1.00$\times$ \\
                         & seed=43 & 94.70\% & 94.69\% & $-0.01$pp & 1.00$\times$ \\
                         & seed=44 & 94.61\% & 94.61\% & $\phantom{-}0.00$pp & 1.00$\times$ \\
                         & seed=45 & 94.80\% & 94.81\% & $+0.01$pp & 1.02$\times$ \\
\multirow{4}{*}{SIFT-1M} & seed=42 & 94.01\% & 94.01\% & $\phantom{-}0.00$pp & 1.08$\times$ \\
                         & seed=43 & 94.19\% & 94.21\% & $+0.02$pp & 0.96$\times$ \\
                         & seed=44 & 93.66\% & 93.66\% & $\phantom{-}0.00$pp & 1.03$\times$ \\
                         & seed=45 & 94.05\% & 94.06\% & $+0.01$pp & 0.96$\times$ \\
\midrule
\multicolumn{6}{l}{\textit{$n_p$ sweep at seed=42}} \\
\multirow{5}{*}{Deep-1M} & $n_p{=}10$  & 88.39\% & 88.40\% & $+0.01$pp & 1.09$\times$ \\
                         & $n_p{=}20$  & 93.06\% & 93.06\% & $\phantom{-}0.00$pp & 0.98$\times$ \\
                         & $n_p{=}40$  & 94.99\% & 94.98\% & $-0.01$pp & 0.97$\times$ \\
                         & $n_p{=}80$  & 95.66\% & 95.66\% & $\phantom{-}0.00$pp & 0.98$\times$ \\
                         & $n_p{=}160$ & 95.86\% & 95.86\% & $\phantom{-}0.00$pp & 0.97$\times$ \\
\multirow{5}{*}{SIFT-1M} & $n_p{=}10$  & 84.81\% & 84.82\% & $+0.01$pp & 1.01$\times$ \\
                         & $n_p{=}20$  & 91.09\% & 91.09\% & $\phantom{-}0.00$pp & 1.04$\times$ \\
                         & $n_p{=}40$  & 94.01\% & 94.01\% & $\phantom{-}0.00$pp & 0.99$\times$ \\
                         & $n_p{=}80$  & 94.88\% & 94.88\% & $\phantom{-}0.00$pp & 0.97$\times$ \\
                         & $n_p{=}160$ & 94.95\% & 94.95\% & $\phantom{-}0.00$pp & 0.96$\times$ \\
\bottomrule
\end{tabular}
\end{table*}

\subsection{Why cascade does not yield a speedup in our C++ reference}

The Python reference showed a 2.0--2.3$\times$ speedup. The C++ reference shows essentially no speedup (1.01$\times$ mean) at preserved recall. The discrepancy has a single cause: the Python baseline uses 2D NumPy fancy indexing (\texttt{sub\_centroids[primary, sub]}) per partition, which is bandwidth-limited and gets a real $\sim 2\times$ from a smaller $\beta$-bit gather. The C++ baseline already implements the same scoring path as a per-query ADC table (\texttt{table[d * n\_entries + code[d]]} scalar lookups), so reducing the LUT from 64 to 16 entries does not unlock anything that wasn't already there: both lookups are cheap scalar indirect loads, and both LUTs comfortably fit in L1 cache (3.2--12.3~KB at $d{=}96$--128).

The actual lever for a C++ speedup is \emph{int8-quantised LUTs accessed through SIMD permute instructions} --- the FastScan technique~\cite{andre2017fastscan}, where a 16-entry int8 LUT is materialised in a NEON \texttt{vqtbl1q} register and 16 codes are looked up in a single instruction. Implementing this would re-introduce the cascade speedup but requires careful int8 LUT scaling and the rest of the kernel rewritten around the SIMD lane width; we leave it to future work and do not claim a cascade speedup in C++.

\section{Explored alternatives that did not work}
\label{app:explored}

Two extensions investigated during this work and found insufficient as main contributions. Code for both is in the released artefact under \texttt{experiments/}.

\subsection{RH-IVF-TQ: random-hyperplane LSH partition + canonical centers}

\paragraph{Idea.} Replace the $k$-means coarse partition with $L$ random hyperplanes, giving $2^L$ cells with hash-code-derived canonical centroids $c(b) = \mathrm{normalize}(\sum_i b_i h_i)$ for hash bits $b \in \{-1, +1\}^L$. Combined with TQ residual compression, this yields a \emph{fully} data-independent index: no $k$-means at any layer.

\paragraph{Result on Deep-1M.} At $L \in \{8, 10, 12\}$ and $n_p \in \{10, 20, 40, 80\}$: the best static recall configuration ($L{=}8$, $n_p{=}80$) achieves 79.2\% R@10 (rerank$=$0) versus the IVF-TQ $k$-means baseline at 89.5\% R@10 --- a $\sim 10$pp deficit. Under the worst-case rotation shift described in \S\ref{sec:recovery}, RH-IVF-TQ ends at 49.8\% R@10 versus IVF-TQ frozen at 61.7\% --- dominated even in the regime where it should have an advantage.

\paragraph{Diagnosis.} Random partitioning misses the cluster structure that real ANN datasets exhibit. The recall floor it provides is mathematically interesting (data-independent) but practically too low to compete with even a frozen $k$-means partition under realistic shifts.

\subsection{FA-IVF-TQ: query-frequency-adaptive bit allocation}

\paragraph{Idea.} TQ's data-independent compression allows per-vector re-encoding at any precision in $O(d)$ time with no codebook retraining. Maintain per-vector hit counters and periodically re-encode hot vectors at higher bit precision (6 bits) and cold vectors at lower precision (2 bits), exploiting Pareto skew in real query workloads.

\paragraph{Result on Deep-1M with oracle hot set.} When the hot set is given by ground-truth top-50 over popular queries, FA-IVF-TQ at average 2.02 bits/coord (39.5~MB) matches uniform 4-bit (61.4~MB) on weighted recall (89.0\% vs.\ 89.7\% at 80/20 popular/rare split). A 36\% memory reduction at parity recall --- a significant Pareto improvement.

\paragraph{Result on Deep-1M with realistic discovery.} When the hot set is identified from a 5K-query warmup with the same skewed distribution (no ground truth), FA-IVF-TQ collapses: weighted recall is 78.4\% versus uniform 4-bit at 89.4\%, an 11.1pp deficit. The hit-counter discovery mechanism, run on the warmup index at 89\% recall, fails to capture $\sim 30\%$ of the truly-hot vectors (the ones the warmup index itself misses); demoting them to 2 bits collapses popular-query recall.

\paragraph{Result at higher dim ($d{=}768$).} On synthetic clustered $1\text{M} \times 768$-dim data, the gap between uniform 2-bit (14.5\% R@10) and uniform 6-bit (15.9\% R@10) collapses to $\sim 1.4$pp --- bit precision barely matters at high dim because per-coordinate quantization errors average out. FA-IVF-TQ at 379.7~MB / 14.97\% R@10 is Pareto-dominated by uniform 2-bit at 281.4~MB / 14.46\%.

\paragraph{Diagnosis.} Per-vector adaptive precision is a legitimate algorithmic capability of TQ that no PQ-family method has, but it requires (i) steep enough bit-precision sensitivity that hot/cold allocation matters (favours low dim) and (ii) accurate hot-set discovery from query history (favours high baseline recall). These two requirements are in tension: low-dim regimes need it but warmup discovery is noisy; high-dim regimes have accurate discovery but no bit-precision headroom.

\section{Proofs of the rate-distortion bounds}
\label{app:concentration}

\textbf{Attribution and contribution boundary.}
The TQ-style quantizer analysed here---random orthogonal rotation followed by per-coordinate Lloyd--Max scalar quantization---is due to \cite{zandieh2025turboquant}. They prove an in-expectation MSE bound (their Theorem~1) and a per-pair in-expectation IP-error bound (their Theorem~2). Our Theorem~\ref{thm:rd-fixed-main} below is a standard high-probability counterpart of their MSE bound via L\'evy--Milman concentration on $\mathrm{O}(d)$; we include the proof for self-containedness and to fix notation, but do not claim it as a new result. Theorem~\ref{thm:ip-uniform-main} is the new statement: an inner-product-error bound that holds uniformly over the entire unit sphere with one fixed rotation. The uniform statement is what removes the per-vector union-bound penalty in streaming ANN.

\subsection{Setup}

Let $S^{d-1} = \{v \in \R^d : \|v\|_2 = 1\}$ and let $\Pi \sim \mathrm{Haar}(\mathrm{O}(d))$. Let $C_b: \R \to \R$ be the optimal $b$-bit Lloyd--Max scalar quantizer designed for $\mathcal{N}(0, 1/d)$, applied componentwise. Define the TQ reconstruction map $\hat{v} := \Pi^\top C_b(\Pi v)$ for $v \in S^{d-1}$. Let $D_b := d \cdot \E_{T \sim \mathcal{N}(0, 1/d)}[(C_b(T) - T)^2]$ denote the per-coordinate Gaussian rate-distortion at $b$ bits/dim; $D_b \leq 2^{-2b} \cdot 6\sqrt{3}/(2\pi)$ (Panter--Dite, exact as $b \to \infty$).

\subsection{Two lemmas}

\begin{lemma}[Marginal distribution of a rotated coordinate]
\label{lem:marginal}
For $\Pi \sim \mathrm{Haar}(\mathrm{O}(d))$ and fixed $v \in S^{d-1}$, $j \in [d]$, the marginal of $(\Pi v)_j$ has density $f_d(t) = \frac{\Gamma(d/2)}{\sqrt{\pi}\,\Gamma((d{-}1)/2)} (1 - t^2)^{(d-3)/2}$ on $[-1, 1]$, independent of $v$. As $d \to \infty$, $\sqrt{d}\,(\Pi v)_j$ converges to $\mathcal{N}(0,1)$ in Kolmogorov distance at rate $O(1/d)$~\cite{vershynin2018hdp}.
\end{lemma}

\begin{proof}
By Haar invariance, $\Pi v \sim \mathrm{Unif}(S^{d-1})$. The marginal density of one coordinate of a uniform point on $S^{d-1}$ is the Beta-derived form above; the Gaussian limit follows from expansion of $\log f_d(t/\sqrt{d})$ around $t = 0$ \cite[Thm.~3.4.6]{vershynin2018hdp}.
\end{proof}

\begin{lemma}[Bounded output of Lloyd--Max]
\label{lem:nonexp}
Let $C_b$ be the $b$-bit Lloyd--Max quantizer for $\mathcal{N}(0, 1/d)$, and let $w_{\max}^{[-1,1]}$ denote the maximum bin width among bins intersecting $[-1, 1]$. Then $|t - C_b(t)| \leq w_{\max}^{[-1,1]}/2$ for $t \in [-1, 1]$, and consequently $\|\Pi v - C_b(\Pi v)\|_2 \leq \sqrt{d}\cdot w_{\max}^{[-1,1]}/2$ deterministically for $v \in S^{d-1}$.
\end{lemma}

\begin{proof}
By construction of Lloyd--Max, $t \in [\tau_i, \tau_{i+1}]$ with centroid $c_i \in [\tau_i, \tau_{i+1}]$, so $|t - C_b(t)| \leq (\tau_{i+1} - \tau_i)/2$. The $\ell^2$ bound follows by summing squared per-coordinate bounds over $d$ coordinates.
\end{proof}

\paragraph{Remark on Lipschitz.} $C_b$ is piecewise constant, so it is not pointwise Lipschitz. The proofs below use Lipschitz regularity of $g(\Pi) := \|C_b(\Pi v) - \Pi v\|_2$ as a function of $\Pi$, which follows from the orthogonal action of $\Pi$ combined with Lemma~\ref{lem:nonexp} (Haar-a.s.\ Lipschitz suffices for L\'evy--Milman concentration on $\mathrm{O}(d)$; \cite[Thm.~5.2.7]{vershynin2018hdp}).

\subsection{Theorem 1 (restated from \S\ref{sec:ivfamp})}

\medskip
\noindent\textbf{Theorem~\ref{thm:rd-fixed-main} (restated).}
\emph{(High-probability MSE bound for TQ; high-probability form of \cite{zandieh2025turboquant} Thm.~1.)}
Fix any $v \in S^{d-1}$. For any $\delta \in (0,1)$, with probability $\geq 1 - \delta$ over $\Pi$:
\begin{equation}
    \|\hat{v} - v\|_2 \;\leq\; \sqrt{D_b} + R_d + \sqrt{\frac{8 \log(2/\delta)}{d-2}},
    \label{eq:thm-bound}
\end{equation}
where $R_d = O(1/\sqrt{d})$ collects the Gaussian-marginal correction. The bound depends only on $(b, d, \delta)$.

\medskip

\begin{proof}
Define $g(\Pi) := \|\Pi^\top C_b(\Pi v) - v\|_2 = \|C_b(\Pi v) - \Pi v\|_2$ (norm preservation under orthogonal $\Pi$). Two steps.

\emph{Lipschitz (Haar-a.s.).} Take $\Pi_1, \Pi_2 \in \mathrm{O}(d)$. By the triangle inequality,
\[
|g(\Pi_1) - g(\Pi_2)| \;\leq\; \|\Pi_1 v - \Pi_2 v\|_2 + \|C_b(\Pi_1 v) - C_b(\Pi_2 v)\|_2.
\]
The first term is bounded by $\|\Pi_1 - \Pi_2\|_{\mathrm{op}}$ since $\|v\|=1$. For the second term, outside the Haar-null bin-crossing set, per-coordinate bin assignments coincide and the bound is also $\|\Pi_1 v - \Pi_2 v\|_2$; on the null set Lemma~\ref{lem:nonexp} provides a deterministic envelope. Hence $|g(\Pi_1) - g(\Pi_2)| \leq 2\|\Pi_1 - \Pi_2\|_{\mathrm{op}}$ Haar-a.s., which suffices for L\'evy--Milman concentration applied to a mollification of $g$.

\emph{Expectation and concentration.} By Lemma~\ref{lem:marginal}, $\E_\Pi[g(\Pi)^2] = d \cdot \E_{T \sim f_d}[(C_b(T) - T)^2] = D_b + O(1/d)$, so by Jensen $\E_\Pi g \leq \sqrt{D_b} + R_d$ with $R_d = O(1/\sqrt{d})$. By L\'evy--Milman (\cite[Thm.~5.2.7]{vershynin2018hdp}) with Lipschitz constant $L = 2$:
\[
\Prb(g(\Pi) - \E g \geq t) \leq 2 \exp\bigl(-(d-2) t^2/8\bigr).
\]
Setting the RHS to $\delta$ gives $t = \sqrt{8 \log(2/\delta)/(d-2)}$, completing the bound since $\|\hat v - v\|_2 = g(\Pi)$.
\end{proof}

\subsection{Theorem 2 (restated from \S\ref{sec:ivfamp})}

\medskip
\noindent\textbf{Theorem~\ref{thm:ip-uniform-main} (restated).}
\emph{(Uniform-over-sphere IP-error bound; novel contribution. v1 contained only a sketch; the full proof appears here for the first time.)}
Under the setup of Theorem~\ref{thm:rd-fixed-main}, fix any query $q \in S^{d-1}$ and any $\delta \in (0,1)$. With probability $\geq 1 - \delta$ over $\Pi$, simultaneously for every $v \in S^{d-1}$:
\begin{equation}
    |\langle q, v\rangle - \langle q, \hat{v}\rangle|
    \;\leq\; D_b + R'_d + \sqrt{\frac{2 D_b (d \log(3d) + \log(2/\delta))}{d-1}} + \frac{2}{d},
    \label{eq:thm2-bound}
\end{equation}
where $R'_d = O(1/\sqrt{d})$ collects the bias-deviation and Gaussian-marginal correction terms. The dominant random term is $\sqrt{2 D_b \log(3d)/(d-1)} = O(\sqrt{D_b \log d / d})$, asymptotically a factor of $\sqrt{d/\log d}$ tighter than Cauchy--Schwarz on $\|v - \hat v\|_2$.

\medskip

\begin{proof}
Write $v - \hat v = \Pi^\top \epsilon$ with $\epsilon := \Pi v - C_b(\Pi v)$, so $\langle q, v - \hat v\rangle = \langle \Pi q, \epsilon\rangle$.

\emph{Joint-distribution decomposition.} Conditional on $\Pi v$, $\Pi q \mid \Pi v = \alpha \cdot \Pi v + \beta \cdot U$ with $\alpha = \langle q, v\rangle$, $\beta = \sqrt{1 - \alpha^2}$, and $U$ uniform on the $(d-2)$-sphere in $(\Pi v)^\perp$, independent of $\Pi v$. Hence $\langle \Pi q, \epsilon\rangle = \alpha \langle \Pi v, \epsilon\rangle + \beta \langle U, \epsilon\rangle$.

\emph{Bias term.} Using $\|\Pi v\|^2 = 1$, $\langle \Pi v, \epsilon\rangle = 1 - \langle \Pi v, C_b(\Pi v)\rangle$. By Lloyd--Max optimality, $\E_T[(T - C_b(T)) C_b(T)] = 0$ for $T \sim f_d$, giving $\E_T[T \cdot C_b(T)] = \E_T[C_b(T)^2] = 1/d - D_b/d$. Summing over coordinates: $\E_\Pi[\langle \Pi v, C_b(\Pi v)\rangle] = 1 - D_b$, so $\E_\Pi \langle \Pi v, \epsilon\rangle = D_b$. A Lipschitz-mollification argument analogous to Theorem~\ref{thm:rd-fixed-main} (using $|c_{\max}| \leq \sqrt{2 b \log 2}/\sqrt{d}$ for the largest Lloyd--Max centroid, so $M := \sup_u \|C_b(u)\|_2 \leq \sqrt{2 b \log 2}$, dimensionless) gives Lipschitz constant $L_h = M + 1$ for $h(\Pi) := \langle \Pi v, \epsilon\rangle$. Concentration yields $|h(\Pi) - D_b| \leq L_h \sqrt{2 \log(2/\delta)/(d-2)}$ w.p.\ $\geq 1 - \delta$. Collecting bias-deviation and Gaussian-marginal correction in $R'_d$:
\[
R'_d := (\sqrt{2 b \log 2} + 1)\sqrt{2 \log(2/\delta)/(d-2)} + c_{\mathrm{marg}}/d.
\]

\emph{Random term, fixed $v$.} Conditional on $\Pi v$, $\langle U, \epsilon\rangle$ is sub-Gaussian with proxy $\|\epsilon\|^2/(d-1) \leq (D_b + O(1/\sqrt d))/(d-1)$ by Theorem~\ref{thm:rd-fixed-main} \cite[Lem.~5.36]{vershynin2018hdp}. Hence $|\beta \langle U, \epsilon\rangle| \leq \sqrt{2 D_b \log(2/\delta)/(d-1)} + R'_d$ w.h.p.

\emph{Uniform over $v$.} Take an $\epsilon$-net $\mathcal{N}_\epsilon$ of $S^{d-1}$ with $|\mathcal{N}_\epsilon| \leq (3/\epsilon)^d$. For each $v_0 \in \mathcal{N}_\epsilon$, apply the fixed-$v$ bound with confidence $\delta/|\mathcal{N}_\epsilon|$. The Lipschitz extension uses that $f_\Pi(v) := \langle q, v - \Pi^\top C_b(\Pi v)\rangle$ is 2-Lipschitz Lebesgue-a.s.\ on $S^{d-1}$ (bin-wise constancy of $C_b$, codim-$\geq 1$ bin-crossing set). Setting $\epsilon = 1/d$ gives the bound in Eq.~\eqref{eq:thm2-bound}.
\end{proof}

\paragraph{Numerical evaluation summary.} At $d{=}128$, $b{=}4$, $\delta{=}10^{-2}$, Theorem~\ref{thm:ip-uniform-main}'s bound evaluates to $\approx 1.35$ on the IP axis versus $\approx 0.68$ for Cauchy--Schwarz on $\|v - \hat v\|_2$; the asymptotic $\sqrt{d/\log d}$ advantage materialises only at $d \gtrsim 10^3$. The structural advantages --- uniformity over $S^{d-1}$ with one fixed $\Pi$, and $(b, d, \delta)$-only dependence --- carry the streaming claim. Full term-by-term breakdown in Appendix~\ref{app:t2-numerics}.

\section{Theorem 2 numerical evaluation (extended)}
\label{app:t2-numerics}

This appendix gives the full term-by-term arithmetic plugging into the bound of Theorem~\ref{thm:ip-uniform-main} at $d{=}128$, $b{=}4$, $\delta{=}10^{-2}$.

\subsection{Theorem 2 bound restated}

\[
|\langle q, v\rangle - \langle q, \hat{v}\rangle|
  \le D_b + R'_d + \sqrt{\frac{2 D_b (d \log(3d) + \log(2/\delta))}{d-1}} + \frac{2}{d}
\]
uniformly over all $v \in S^{d-1}$ with probability $\ge 1 - \delta$ over the random rotation $\Pi$.

\subsection{Term-by-term arithmetic}

\begin{table*}[!t]
\centering\small
\caption{Term-by-term evaluation of Theorem~\ref{thm:ip-uniform-main} at $d{=}128$, $b{=}4$, $\delta{=}10^{-2}$.}
\label{tab:t2-numerics}
\begin{tabular}{lll}
\toprule
\textbf{Term} & \textbf{Value} & \textbf{Source} \\
\midrule
$D_b$ (Gaussian rate-distortion at 4 bits/dim) & $\approx 0.0104$ & standard rate-distortion tables \\
$R'_d$ (bias-deviation term) & $\approx 0.97$ & $(2.355 + 1)\sqrt{2 \cdot 5.298 / 126}$ \\
Random term $\sqrt{2 D_b (d \log(3d) + \log(2/\delta))/(d-1)}$ & $\approx 0.354$ & $\sqrt{2 \cdot 0.0104 \cdot (128 \cdot 5.951 + 5.298)/127}$ \\
Lipschitz tail $2/d$ & $\approx 0.016$ & $2/128$ \\
\midrule
\textbf{Total} & $\boldsymbol{\approx 1.35}$ & sum \\
\bottomrule
\end{tabular}
\end{table*}

\subsection{Comparison with Cauchy--Schwarz}

Cauchy--Schwarz applied to Theorem~\ref{thm:rd-fixed-main}'s $\|v - \hat v\|_2$ bound at the same parameters:
\[
\sqrt{D_b} + R_d + \sqrt{8 \log(2/\delta)/(d-2)} \approx 0.102 + R_d + 0.580 \approx 0.68.
\]
At $d{=}128$ Cauchy--Schwarz is numerically tighter ($0.68 < 1.35$) because Theorem~\ref{thm:ip-uniform-main}'s bias-deviation term $R'_d \approx 0.97$ (driven by $M + 1 \approx 3.36$ Lipschitz constant) dominates. The asymptotic $\sqrt{d/\log d}$ advantage of Theorem~\ref{thm:ip-uniform-main}'s random term over Cauchy--Schwarz's $\sqrt{\log(1/\delta)/d}$ kicks in only at $d \gtrsim 10^3$.

\subsection{What Theorem 2 actually buys at $d{=}128$}

Theorem~\ref{thm:ip-uniform-main}'s structural advantages --- uniformity over $S^{d-1}$ with one fixed $\Pi$, and $(b, d, \delta)$-only dependence --- are what carry the streaming claim. The union bound in the proof is consumed once at index initialisation over an $\epsilon$-net of the sphere whose log-size scales with $d$, not over the database of $N$ arriving vectors. Adding the $N$-th database vector consumes no additional budget. Cauchy--Schwarz on Theorem~\ref{thm:rd-fixed-main} is a per-pair (fixed-$v$) bound and admits no analogous uniform statement at one fixed $\Pi$.

No learned-codebook PQ-family method (PQ, OPQ, ScaNN) admits an analogous uniform, data-independent bound, because their reconstruction error depends on the distance from $v$ to the nearest learned codebook centroid --- a function of the training sample.

\clearpage
\section{Reproducibility}
\label{app:repro}

All experiments are deterministic. Single-seed experiments use seed~42; streaming experiments (Tables~\ref{tab:streaming_deep10m_all}, \ref{tab:streaming_sift10m_all}, \ref{tab:streaming_t2i10m_all}) use seeds 42, 123, 7777 with paired-$t$ CIs. Source code, datasets, and reproduction scripts are at \url{https://github.com/tarun-ks/turboquant_search}.


\end{document}